%% file: main.tex
\documentclass[12pt, pdftex]{article} 

\usepackage{times}
\usepackage{fullpage}
\usepackage{url}
\usepackage[utf8]{inputenc}
\usepackage[T1]{fontenc}

\usepackage{graphicx}
\graphicspath{{./image/}}

\usepackage{hyperref}
\hypersetup{colorlinks=true, citecolor=blue, anchorcolor=blue, linkcolor=blue,
    filecolor=blue, menucolor=blue, urlcolor=blue, plainpages=false,
    pdfpagelabels}

\usepackage[numbers]{natbib}
\renewcommand{\cite}{\citep}
\usepackage{enumerate}
\usepackage{amsmath}
\usepackage{amsfonts}
\usepackage{amssymb}
\usepackage{amsthm}
\usepackage{mathtools}
\usepackage{subfig}
\usepackage{booktabs}
\usepackage[vlined, ruled]{algorithm2e}

\usepackage{tikz}
\newcommand*\circled[1]{\tikz[baseline=(char.base)]{
                    \node[shape=circle,draw,inner sep=0.5pt] (char) {#1};}}

\def \epsilon{\varepsilon}
\let\tilde\widetilde
\let\hat\widehat

\newcommand{\abs}[1]{|#1|}

\newcommand{\norm}[1]{\left\|#1\right\|}

\newcommand{\set}[1]{\left\{#1\right\}}
\newcommand{\T}{\top}

\DeclarePairedDelimiter\ip{\langle}{\rangle}

\newcommand{\C}{\mathcal{C}}
\newcommand{\R}{\mathbb{R}}

\newcommand{\A}{\mathcal{A}}
\newcommand{\tS}{\mathcal{S}}
\newcommand{\PP}{\mathcal{P}}

\renewcommand{\r}{r}

\newcommand{\overbar}[1]{\mkern 1.5mu\overline{\mkern-2.5mu#1\mkern-1mu}\mkern 1.5mu}
\newcommand{\Zbar}{\overbar{Z}}

\newcommand{\xbar}{\bar{x}}
\newcommand{\Ztilde}{\tilde{Z}}

\newcommand{\Zstar}{Z^\star}
\newcommand{\Xstar}{X^\star}
\newcommand{\Ystar}{Y^\star}

\newcommand{\Ustar}{U^\star}
\newcommand{\Vstar}{V^\star}
\newcommand{\Sstar}{\Sigma^\star}
\newcommand{\Sstarhalf}{{\Sigma^\star}^\frac{1}{2}}
\newcommand{\Shalf}{\Sigma^\frac{1}{2}}
\newcommand{\sstar}{\sigma^\star}

\DeclareMathOperator{\rank}{\text{rank}}
\DeclareMathOperator{\trace}{\text{tr}}
\DeclareMathOperator*{\argmin}{arg\,min}

\def\OmegaY{\underline{\Omega}}
\def\nmin{n_1 \land n_2}

\newtheorem{theorem}{Theorem}
\newtheorem{lemma}{Lemma}
\newtheorem{corollary}{Corollary}
\newtheorem{defn}{Definition}

\title{\Large\bf 
Convergence Analysis for Rectangular Matrix Completion Using
Burer-Monteiro Factorization and Gradient Descent}

\author{
\normalsize Qinqing Zheng \quad John Lafferty\\
\normalsize University of Chicago \\
}
\date{\normalsize \today}

\begin{document}
\maketitle

\input{abs}
\input{intro}

\input{lift}
\input{rst}

\input{related}
\input{expr}
\input{conclude}

\input{ack}

\appendix
\input{app_ingred}
\input{app_init}

\input{app_regu}

\input{app_conv}

\bibliography{main}
\bibliographystyle{plainnat}
\allowdisplaybreaks

\end{document}

%% file: abs.tex
\begin{abstract}
    We address the rectangular matrix completion problem by lifting the unknown
    matrix to a positive semidefinite matrix in higher dimension, 
    and optimizing a nonconvex objective over the semidefinite factor
    using a simple gradient descent scheme.
    With $O( \mu r^2 \kappa^2 n \max(\mu, \log n))$ random observations of a $n_1 \times n_2$
    $\mu$-incoherent matrix of rank $r$ and condition number $\kappa$, where $n = \max(n_1, n_2)$, 
    the algorithm linearly converges to the global optimum with
    high probability.
\end{abstract}

%% file: intro.tex
\section{Introduction}
\label{sec:intro}

A growing body of recent research is shedding new light on the role of
nonconvex optimization for tackling large scale problems in machine
learning, signal processing, and convex programming.  This work is
developing techniques that help to explain the surprising
effectiveness of relatively simple first-order algorithms for certain
nonconvex optimizations.

When applied to problems that can be formulated as semidefinite
programs, these techniques can often be viewed as part of a framework
proposed by \citet{BurMon03}.  The Burer-Monteiro technique is based
on factoring the semidefinite variable, and applying classical
optimization techniques to the resulting nonconvex objective over the
factor. While worst-case complexity considerations imply that such an
approach cannot succeed in general, a series of recent
papers~\cite{CanLiSol14,ZheLaf15, TuBocSim16, CheWai15, BhoKyrSan15}
has shown the strategy to be remarkably effective for a number of
problems of practical interest, with analytical convergence guarantees
and strong empirical performance.

In this paper, we enlarge the collection of problems
to which the Burer-Monteiro technique can be successfully applied, by
analyzing the convergence properties of gradient descent applied to
the problem of rectangular matrix completion from incomplete
measurements.  The standard matrix completion problem asks for
the recovery of a low rank matrix 
$\Xstar \in \R^{n_1 \times n_2}$ given only a small fraction 
of observed entries.  Let $\Omega$ be the set of $m$ indices 
of the observed entries. Fixing a target rank $r \ll \min(n_1, n_2)$,
the natural, but nonconvex objective is
\begin{equation}
    \label{eq:opt_rank_min}
    \begin{aligned}
        & \min_{X \in \R^{n_1 \times n_2}} && \rank(X) \\
        & \text{subject to} && X_{ij} = \Xstar_{ij}, \; (i,j) \in \Omega.
    \end{aligned}
\end{equation}
In order for this problem to be well-posed, it is important
to understand when $\Xstar$ is identifiable and, in particular, 
the unique minimizer of \eqref{eq:opt_rank_min}.
Moreover, because the problem is in general NP-hard, 
it is essential to identify tractable families of instances,
together with efficient algorithms 
having global convergence guarantees.

In the current work, we apply the factorization technique
by ``lifting'' the matrix $\Xstar$ to a positive semidefinite matrix $Y^\star \in \R^{(n_1 + n_2) \times
    (n_1 + n_2)}$ in higher dimension. 
Lifting is an established method that recasts 
vector or matrix estimation problems in terms of positive semidefinite matrices with
special structure.  It has been applied to sparse eigenvector approximation \cite{sparse_pca_sdp} and
phase retrieval \cite{phaselift_1}, where the lifted matrix is of rank
one. 
As explained in detail below, we can construct $Y^\star$ to be of the same rank as
$\Xstar$, thus obtaining a factorization
$Y^\star = \Zstar {\Zstar}^\T$ for some $\Zstar \in \R^{(n_1 + n_2)
  \times r}$, and transforming the original matrix completion problem
into the problem of recovering the semidefinite factor $\Zstar$.
We formulate this as minimizing a nonconvex objective $f(Z)$, to which
we apply a gradient descent scheme, using a particular spectral
initialization. Our analysis of this algorithm establishes a lower
bound on the number of matrix measurements that are sufficient to 
guarantee identifiability of the true matrix and
geometric convergence of the gradient descent algorithm,
with explicit bounds on the rate.

In the following section we give a full description of our approach.  
Our theoretical results are presented in Section~\ref{sec:rst}, with detailed proofs contained in
the appendix. Our analysis subsumes the case where $\Xstar$ is positive
semidefinite. In Section~\ref{sec:related} we briefly review related
work.  The experimental results are presented in
Section~\ref{sec:expr}, and we conclude with a brief discussion of
future work in Section~\ref{sec:conclude}.


%% file: lift.tex
\def\sfrac#1#2{{#1}/{#2}}

\section{Semidefinite Lifting, Factorization, and Gradient Descent}
\label{sec:lift}

For any $(n_1 + n_2) \times r$ matrix $Z$, we will use $Z_{(i)}$ to denote its
$i$th row, and $Z_U$ and $Z_V$ to denote the top $n_1$ and bottom $n_2$
rows.  The operator, Frobenius and $\ell_\infty$ norm of matrices are
denoted by $\norm{\cdot}$, $\norm{\cdot}_F$ and $\norm{\cdot}_\infty$,
respectively. We define $\norm{Z}_{2, \infty} = \max_{i}
\norm{Z_{(i)}}_2$ as the largest $\ell_2$ norm of its rows, and
similarly $\norm{Z}_{\infty, 2} = \max \set{ \norm{Z}_{2, \infty},
  \norm{Z^\T}_{2, \infty} }$.  
Let $\PP_\Omega: \R^{n_1 \times n_2} \rightarrow \R^{n_1 \times n_2}$ be the operator where 
\begin{equation}
    \label{eq:PP_Omega}
    \PP_\Omega(X)_{ij} =
                \begin{cases}
                    X_{ij} & \text{if}\; (i,j) \in \Omega, \\
                    0 & \text{otherwise}.
                \end{cases}
\end{equation}

In this paper, we focus on completing an incoherent or
``non-spiky'' matrix $\Xstar$. 
With $\Ustar \Sstar \Vstar$ denoting the
rank-$r$ SVD of $\Xstar$, we assume $\Xstar$ is $\mu$-\textit{incoherent}, as
defined below.
\begin{defn}
    \label{defn:incoh}
The matrix $\Xstar$ is $\mu$-incoherent with respect to the canonical
basis if its singular vectors satisfy
\begin{equation}
    \label{eq:defn_incoh}
\norm{\Ustar}_{2, \infty} \leq \sqrt{\frac{\mu r}{n_1}},
 \quad \norm{\Vstar}_{2, \infty} \leq \sqrt{\frac{\mu r}{n_2}},
\end{equation}
where $\mu$ is a constant.\footnote{Note that $\mu \geq 1$, since
$r = \|\Ustar \|^2_F = \sum_{i \in [n_1]} \norm{\Ustar_{(i)}}^2_2 \leq \mu r$.}
\end{defn}
Our main interest is the uniform model where $m$ entries of $\Xstar$ are
observed uniformly at random, though we shall analyze a Bernoulli sampling model, where each
entry of $\Xstar$ is observed with probability $p = \sfrac{m}{n_1 n_2}$. One can
transfer the results back to the uniform model, as the probability
of failure under the
uniform model is at most twice that under the Bernoulli model; see
\cite{CanRec09, CanTao10}.

Using the rank-$r$ SVD of $\Xstar$, we can lift $\Xstar$ to
\begin{equation}
    \label{eq:lift}
     Y^\star = \begin{bmatrix} \Ustar \Sstar {\Ustar}^\T & \Xstar \\ {\Xstar}^\T
     & \Vstar \Sstar {\Vstar}^\T \end{bmatrix} = \Zstar {\Zstar}^\T, 
    \quad \text{where} \;
    \Zstar = \begin{bmatrix} \Ustar \\ \Vstar\end{bmatrix} {\Sstar}^\frac{1}{2}.
\end{equation}
The symmetric decomposition of $\Ystar$ is not unique; our goal is to find 
a matrix in the set
\begin{equation}
    \tS = \set{\Ztilde \in \R^{(n_1 + n_2)\times \r} \; | \; \Ztilde = \Zstar
        R \;\; \text{for some $R$ with $RR^\T = R^\T R = I$}},
\end{equation}
since for any $\Ztilde \in \tS$ we have $\Xstar = \Ztilde_U {\Ztilde_V}^\T$.
Let $\underline{\Omega}$ denote the corresponding observed entries of $\Ystar$,
and consider minimization of the squared error
\begin{equation}
    \label{eq:prob_after_lift}
        \min_{Z} \; \frac{1}{2p} \sum_{(i,j) \in \OmegaY} ( ZZ^\T_{ij} -
        \Ystar_{ij} )^2 = \min_{Z} \frac{1}{2p} \norm{ \PP_{\OmegaY} ( ZZ^\T - \Ystar )
        }^2_F.
\end{equation}
Note that $\Ystar$ is not the unique minimizer of
\eqref{eq:prob_after_lift}, nor is it the only possible 
positive semidefinite lifting of $\Xstar$. For example, let $P$ be an
$r\times r$ nonsingular
matrix, and form the matrices
\begin{equation}
    Z' = \begin{bmatrix} 
        \Ustar {\Sstar}^\frac{1}{2} P \\
        \Vstar {\Sstar}^\frac{1}{2} P^{-1}
        \end{bmatrix}
          \qquad
    Y' = \begin{bmatrix} 
        \Ustar {\Sstar}^\frac{1}{2} P^2 {\Sstar}^\frac{1}{2}  {\Ustar}^\T  & \Xstar \\
        {\Xstar}^\T & \Vstar {\Sstar}^\frac{1}{2} P^{-2} {\Sstar}^\frac{1}{2}
        {\Vstar}^\T
        \end{bmatrix}.
\end{equation}
Since $\OmegaY$ does not contain any entry in the top-left or bottom-right
block, $Y'$ is also a minimizer of \eqref{eq:prob_after_lift}.
Thus, the solution set of the lifted problem is much larger than the
set $\tS$ of actual interest. For the sake of simple analysis,
we shall focus on exact recovery
of $\Ystar$ only, and thus impose an additional regularizer 
to align the column spaces of $Z_U$ and $Z_V$, as in \cite{TuBocSim16}. The regularized loss is 
\begin{equation}
    \label{eq:obj_reg}
f(Z) = \frac{1}{2p} \norm{ \PP_{\OmegaY} ( ZZ^\T - \Ystar) }^2_F +
\frac{\lambda}{4} \norm{ Z^\T D Z}^2_F,
\quad \text{where\ } D = \begin{bmatrix}  I_{n_1} & 0 \\ 0 & -I_{n_2} \end{bmatrix}.
\end{equation}
While this apparently introduces an extra tuning parameter, our analysis
establishes linear convergence of the projected gradient descent algorithm when $\lambda = \frac{1}{2}$,
and thus one may treat $\lambda$ as a fixed number.

It is discussed in \cite{CheWai15} that one needs to ensure the iterates stay
incoherent. Let $\C$ be the set of incoherent matrices
\begin{equation}
\label{feasset}
\C = \left\{ Z : \norm{Z}_{2, \infty} \leq \sqrt{ \frac{2 \mu r} {n_1 \land n_2} }\norm{Z^0}\right\}
\end{equation}
where we assume $\mu$ is known and $Z^0$ will be determined.

Our algorithm is simply
gradient descent on $f(Z)$, with projection onto $\C$.
Let $M = p^{-1} \PP_\Omega(UV^\T - \Xstar)$.
Then the gradient of $f$ is given by
\begin{equation}
    \begin{aligned}
        \nabla f (Z) & = && \begin{bmatrix}  0 & M \\ M^\T & 0 \\ \end{bmatrix} Z + \lambda D Z Z^\T D Z.
    \end{aligned}
\end{equation}
The projection $\PP_{\C}$ to the feasible set $\C$ has closed form
solution, given by row-wise clipping:
\begin{equation}
    \label{eq:proj_C}
    \PP_{\C} (Z)_{(i)} =
    \begin{cases}
    Z_{(i)} &  \text{if} \; \norm{Z_{(i)}} \leq \sqrt{ \frac{2 \mu r} {\nmin} }\norm{Z^0},\\
       \frac{ Z_{(i)} }{ \norm{ Z_{(i)} } } \cdot \sqrt{ \frac{2 \mu r} {\nmin} }\norm{Z^0} & \text{otherwise.} 
    \end{cases}
\end{equation}

Note that $X^0 \equiv p^{-1} \PP_\Omega(\Xstar)$ is an unbiased estimator
of $\Xstar$ under the Bernoulli model. To initialize, we thus construct $Z^0$ from the top
rank-$r$ factors of $X^0$. This leads to the following algorithm.
\begin{algorithm}[hb]
    \caption{Projected gradient descent for matrix completion}
    \label{alg:gd1}
    \SetKwInOut{Input}{input}
    \Input{ $\Omega$, $\set{ \Xstar_{ij}:  (i,j) \in \Omega }$, $m$, $n_1$, $n_2$, $r$, $\lambda$, $\eta$ }
    \textbf{initialization} \\
    \hspace{0.5cm} $p = m / n_1 n_2$\\
    \hspace{0.5cm} $U^0 \Sigma^0 {V^0}^\T = $ rank-$r$ SVD of
    $p^{-1} \PP_\Omega(\Xstar)$\\
    \hspace{0.5cm} $Z^0 = [U^0 {\Sigma^0}^{\frac{1}{2}}; V^0 {\Sigma^0}^{\frac{1}{2}}]$\\
    \hspace{0.5cm} $Z^1 = \PP_\C(Z^0)$\\
    \hspace{0.5cm} $k \leftarrow  1$\\
    \Repeat {convergence} {
        $M^k = p^{-1} \PP_\Omega( Z^k_U {Z^k}_V^\T - \Xstar)$ \\
        \vspace{3pt}
        $\nabla f(Z^k) = \begin{bmatrix}  0 & M^k \\ {M^k}^\T & 0 \\ \end{bmatrix} Z^k +
        \lambda D Z^k {Z^k}^\T D Z^k.$ \\
        \vspace{3pt}
        $Z^{k+1} = \PP_\C\left( Z^k - \dfrac{\eta}{\norm{Z^0}^2} \nabla f(Z^k) \right)$\\
        $k \leftarrow k+1$
    }
    \SetKwInOut{Output}{output}
    \Output{$\hat{Z} = Z^k, \hat{X} = Z^k_U {Z^k_V}^\T$. }
\end{algorithm}

\textit{Remarks.} (i) The step size $\eta$ is normalized by $\norm{Z^0}^2$. Our analysis 
will establish linear convergence when taking step sizes of the form
$\eta / \sstar_1$, where $\eta$ is a sufficiently small constant. 
We replace $\sstar_1$ by $\norm{Z^0}^2$ in the actual algorithm since it is
unknown in practice. (ii) The feasible set \eqref{feasset}
depends on $\norm{Z^0}$ as well. Under the above spectral initialization, our analysis shows that when $p
\geq O(\sfrac{\mu \kappa^2 r^2 \log n}{\nmin})$, the term $\sqrt{\frac{2 \mu r}{\nmin}
}\norm{Z^0}$ is an upper bound of $\norm{\Zstar}_{2, \infty}$ with
high probability (see Corollary~\ref{coro:proj_const} below). This means $\tS$
is a subset of $\C$. Note that this does not change the global
optimality of $\Zstar$ and its equivalent elements, since $f(\Zstar) = 0$.
In practice, we find that the iterates of our algorithm remain incoherent, so
that one may drop the projection step.
(iii) The column space regularizer \eqref{eq:obj_reg} is needed in our analysis.
We also found that when $\lambda = 0$, our algorithm typically
converges to another PSD lifted matrix of $\Xstar$, with minor difference from $\Ystar$ 
in the top-left and bottom-right blocks.

In the following section we state and sketch a proof of our main
convergence result for this algorithm.

%% file: rst.tex
\def\sfrac#1#2{{#1}/{#2}}
\section{Main Result: Convergence Analysis}
\label{sec:rst}

\begin{theorem}
\label{thm:main}
Suppose that $\Xstar$ is of rank $r$,
with condition number $\kappa = \sstar_1/\sstar_r$,  and $\mu$-incoherent as
defined in Definition~\ref{defn:incoh}. Suppose further that we observe $m$ entries of $\Xstar$
chosen uniformly at random.
Let $Y^\star = \Zstar {\Zstar}^\T$ be the lifted matrix as in
\eqref{eq:lift} and write $n = \max(n_1, n_2)$.
Then there exist universal constants $c_0, c_1, c_2, c_3$ such that if 
\begin{equation}
 m \geq c_0 \mu r^2 \kappa^2 \max(\mu, \log n) n, 
\end{equation}
then with probability at least $1 - c_1 n^{-c_2}$ 
the iterates of Algorithm~\ref{alg:gd1} 
converge to $\Zstar$ geometrically,
when using regularization parameter $\lambda = 1/2$, correctly specified input rank
$r$, and constant step size $\eta / \sstar_1$ with $\eta \leq \displaystyle {c_3}/{\mu^2 r^2 \kappa }$.
\end{theorem}

We shall analyze the Bernoulli sampling model, as justified in
Section~\ref{sec:lift}. 
Let us define the distance to $\Zstar$ in terms of the solution set $\tS$. 
\begin{defn} Define the distance between $Z$ and $\Zstar$ as
    \[ d(Z, \Zstar) = \min_{\Ztilde \in \tS} \big \| Z - \Ztilde \big \|_F
         = \min_{RR^\T = R^\T R =I} \norm{Z - \Zstar R}_F. \]
\label{def:dist}
\end{defn} 
The next theorem establishes the global convergence of
Algorithm~\ref{alg:gd1}, assuming that the input rank is correctly specified. 
The proof sketch is given in the next subsection.
\begin{theorem}
\label{thm:recovery}
There exist universal constants $c_0, c_1, c_2$ such that if 
$p \geq \dfrac{c_0 \mu r^2 \kappa^2 \log n}{\nmin}$,
with probability at least $1 - c_1 n^{-c_2}$, the initialization $Z^1 \in \C$ satisfies
    \begin{equation}
        \label{eq:rst_init}
            d(Z^1, \Zstar) \leq \frac{1}{4} \sqrt{\sstar_r}.
    \end{equation}
Moreover, there exist universal constants
$c_3,c_4, c_5, c_6$ such that if $p \geq \dfrac{c_3 \max(\mu r^2 \kappa^2  , \mu r \log n)}{\nmin}$,
when using constant step size $\eta / \sstar_1 $
with $\eta \leq \dfrac{c_4}{\mu^2 r^2 \kappa }$ and initial value
$Z^1 \in \C$ obeying \eqref{eq:rst_init}, the $k$th step of Algorithm~\ref{alg:gd1} with
$\lambda = 1/2$ satisfies
\[ d(Z^k, \Zstar) \leq \frac{1}{4} \left(1 - \frac{99}{256} \cdot \frac{\eta}{\kappa} \right)^{k/2} \sqrt{ \sstar_r } \]
with probability at least $1 - c_5n^{-c_6}$.
\end{theorem}
\textit{Remarks.} 
\begin{itemize}
    \item[(i)] After each update, the distance of our iterates to $\Zstar$ is reduced by at least a factor
    of $1 - O(1 / \mu^2 r^2 \kappa^2)$. 
    \item[(ii)] Hence, the output $\hat{Z}$
satisfies $d(\hat{Z}, \Zstar) \leq \epsilon$ after at most
$\left \lceil 2 \log^{-1}\left( \sfrac{1}{(1 -\frac{99}{256} \cdot \frac{\eta}{\kappa})} \right)
 \log\left( \sfrac{\sqrt{\sstar_r} }{4\epsilon}\right) \right\rceil $ iterations.
 \end{itemize}

\subsection{Proof Sketch}
Our proof idea is of the same nature as the analysis in \cite{CanLiSol14, ZheLaf15}. 
We show two appealing properties when sufficient entries are observed.
First, our spectral initialization produces a starting point within
the $O(\sstar_r)$ neighborhood of the solution set.
\begin{lemma}
    \label{thm:init}
    There exist universal constants $c, c_1, c_2$, such that if $ p \geq \dfrac{c
    \mu r^2 \kappa^2 \log n}{\nmin}$ then with probability at least $1 - c_1 n^{-c_2}$,
    \[ d(Z^1, \Zstar) \leq d(Z^0, \Zstar) \leq \frac{1}{4} \sqrt{ \sstar_r }. \]
\end{lemma}
To demonstrate this, we exploit the concentration around the mean of
$p^{-1}\PP_\Omega(\Xstar)$. See Appendix~B for the proof. 
Using this lemma, we can immediately show that $\Zstar$ and all
other elements of $\tS$ are contained in the feasible set \eqref{feasset}.
\begin{corollary}
    \label{coro:proj_const} With probability at least $1 - c_1 n^{-c_2}$,
    $\norm{\Zstar}_{2, \infty} \leq \sqrt{\frac{2\mu r}{\nmin}}\norm{Z^0}. $
\end{corollary}

The second crucial property is that $f(Z)$ is well-behaved within the
$O(\sqrt{\sstar_r})$ neighborhood, so that the iterates move closer to the optima in every
iteration. The key step is to
set up a \textit{local regularity condition} \cite{CanLiSol14} similar to Nesterov's conditions
\cite{Nes04}.
\begin{defn}
    \label{def:RC}
Let $\Zbar = \argmin_{\Ztilde \in \tS} \big \| Z - \Ztilde \big \|_F $ denote the matrix
closest to $Z$ in the solution set.
    \label{def:rc}
    We say that $f$ satisfies the \textit{regularity condition} $RC(\epsilon, \alpha, \beta)$ if
        there exist constants $\alpha$, $\beta$ such that for any $Z \in \C$
        satisfying $d(Z, \Zstar) \leq \epsilon$, we have
\[  \ip{\nabla f(Z), Z - \Zbar } \geq \frac{1}{\alpha} \sstar_r \norm{Z - \Zbar
    }^2_F + \frac{1}{\beta \sstar_1} \norm{\nabla f(Z)}^2_F. \]
\end{defn}
Using this condition, one can show the iterates converge linearly to the optima if we start
close enough to $\Zstar$.
\begin{lemma}
    \label{thm:linear_conv}
Consider the update $Z^{k+1} = \PP_\C \left(Z^k - \dfrac{\mu}{\sstar_1} \nabla f(Z^k)\right)$.
If $f$ satisfies $RC(\epsilon, \alpha, \beta)$, $d(Z^{k},\Zstar)
\leq \epsilon$ and $0 < \mu \leq \min(\alpha /2, 2/\beta)$, then 
\[
    d(Z^{k+1}, \Zstar) \leq \sqrt{ 1 - \frac{2\mu }{\alpha \kappa} } d(Z^k, \Zstar).
\]
\end{lemma}
The following lemma illustrates the local regularity of $f(Z)$. 
Nesterov's criterion is established upon strong convexity and strong smoothness
of the objective. Here we show analogous \textit{curvature} and
\textit{smoothness} conditions holds for $f(Z)$ locally -- within the
$O(\sqrt{\sstar_r})$ neighborhood -- with high probability. 
Interestingly, we found that to show the local curvature condition holds, it suffices
to set $\lambda = \frac{1}{2}$.
The proof can be found in Appendix~C, for which we have generalized 
some technical lemmas of \cite{CheWai15}.

\begin{lemma}
    \label{thm:rc}
    Let the regularization constant be set to $\lambda = \frac{1}{2}$.
    There exists universal constant $c, c_1, c_2$, such that 
    if $p \geq \dfrac{c \max(\mu^2 r^2 \kappa^2, \mu r \log n)}{\nmin}$, 
    then $f$ satisfies $RC(\frac{1}{4}\sqrt{\sstar_r},
    512/99, 13196 \mu^2 r^2 \kappa )$, with probability at least $1-c_1n^{-c_2}$.
\end{lemma}

%% file: related.tex
\section{Related Work}
\label{sec:related}
Matrix completion is one instance of the general low rank linear inverse problem
\begin{equation}
    \label{eq:low_rank_linear_inverse}
    \text{find} \; X \; \text{of minimum rank such that } \A(X) = b,
\end{equation}
where $\A$ is an affine transformation and $b = \A(\Xstar)$ is the measurement
of the ground truth $\Xstar$. 
Considerable progress has been made towards algorithms for recovering $\Xstar$
including both convex and nonconvex approaches. 
One of the most popular methods is nuclear norm
minimization, a convenient convex relaxation of rank minimization.
It was first proposed in \cite{fazel:02,RecFazPar10}, and analyzed
under a certain \textit{restricted isometry property} (RIP). 
Subsequent work clarified the conditions
for reconstruction, and studied recovery guarantees
for both exact and approximately low rank matrices, with or without noise~\cite{CanRec09, CanTao10, NegWai12, Che15}.
One significant advantage for this approach is its near-optimal sample
complexity. Under the same incoherence assumption as ours,
\citet{Che15} establishes the currently best-known lower bound of $O(\mu r n \log^2 n)$
samples. Using a closely related notion of incoherence,
\citet{NegWai12} show that if
$\Xstar$ is ``$\alpha$-nonspiky'' with 
$\frac{\norm{\Xstar}_\infty}{\norm{\Xstar}_F} \leq \frac{\alpha}{\sqrt{n_1 n_2}}$,
then $O(\alpha^2 r n \log n)$ samples are sufficient for
exact recovery. However, convexity and low sample complexity aside, in practice the power of nuclear norm
relaxation is limited due to high computational cost. The popular
algorithms for nuclear norm minimization are proximal methods that perform iterative singular value
thresholding \cite{CaiCanShe10, TomHayKas10}. However, such algorithms
don't scale to large instances because the per-iteration SVD is
expensive.

Another popular convex surrogate for the rank function is the max-norm
\cite{SreRecJaa04, FoySre11}, given by $\norm{X} = \min_{X = UV^\T}
\norm{U}_{2,\infty} \norm{V}_{2,\infty}$.  For certain types of
problems, the max-norm offers better generalization error bounds than
the nuclear norm \cite{SreShr05}.  But practically solving large scale
problems that incorporate the max-norm is also non-trivial.  In 2010,
\citet{LeeRecSre10} rephrased the max-norm constrained problem as an
SDP, and applied Burer-Monteiro factorization. Although this ends up
with an $\ell_{2, \infty}$ constraint similar to ours \eqref{feasset},
we emphasize that the constraint plays a different role in our
setting.  While \cite{SreRecJaa04, LeeRecSre10} use it to promote low
rank solutions, our purpose is to enforce incoherent solutions; and
experimental results suggest that one can drop it. Moreover, the
convergence of projected gradient descent for this problem was not
previously understood.

In a parallel line of work, the problem of developing techniques that exactly solve nonconvex
formulations has attracted significant recent research attention.
In chronological order, \citet{KesMonOh09} proposed a manifold gradient method
for matrix completion. They factorize $\Xstar = \Ustar \Sstar {\Vstar}^\T$, where $\Ustar \in \R^{n_1 \times r},
U^\T U = n_1 I_{n_1}$ and 
$\Vstar \in \R^{n_2 \times r}, V^\T V = n_2 I_{n_2}$.  Similar to our
definition of $\tS$, the equivalence classes of $\Ustar$ and
$\Vstar$ are Grassmann manifolds of $r$ dimensional subspaces. The authors then minimize
the nonconvex objective $F(U, V) = \min_{S \in \R^{r \times r}} \norm{\PP_\Omega(U S V^\T -
\Xstar)}^2_F$ over the manifolds. In each iteration, $U$ and $V$ are updated
along their manifold gradients, followed by the update of
the optimal scaling matrix $S$. This algorithm can exactly exactly reconstruct
the matrix, though the convergence rate is unknown. 
However, its per-iteration update also has high computational
complexity, see 
Section~\ref{sec:expr} for details. There are other manifold optimization methods for
matrix completion including \cite{BouAbs2011, MisMeyBac2013, Van2013}.

In the same year, \citet{JaiMekDhi10} 
suggested minimizing the
squared residual $\norm{\A(X) - b}^2$ under a rank constraint $\rank(X) \leq r$. While this
constraint is nonconvex, projection onto the feasible set 
can be computed using low rank SVD. 
Under certain RIP assumption on $\A$, Jain et al. establish the global
convergence of projected gradient descent for
\eqref{eq:low_rank_linear_inverse}. This algorithm is named Singular Value Projection
(\texttt{SVP}). Yet in the setting of completion, only
experimental support for the effectiveness of \texttt{SVP} is provided.
More importantly, \texttt{SVP} also suffers from expensive
per-iteration SVD for large scale problems.  

\citet{Keshavan2012, JaiNetSan13}
further analysed the alternating minimization procedure 
for \eqref{eq:low_rank_linear_inverse}. \texttt{AltMin}
factorizes $X = UV^\T$ where $U \in \R^{n_1 \times r}$ and $V \in
\R^{n_2 \times r}$, and alternately solves $\norm{\A(UV^\T) -
b}^2_2$ over $U$ and $V$, while fixing the other factor.
The authors obtain sample complexity bounds with $r\kappa^8$, $r^7{\kappa}^6$ dependency,
respectively. In 2014, \citet{Har14} improved 
the bounds to $r^2\kappa^2$. Notably, all these works assume the use
of \textit{resampling}---independent sequences of samples $\Omega_k, k=1, 2, \ldots$. In other words,
in every iteration we can sample the true matrix under a certain Bernoulli
model independently. However, in practice $\Omega$ is usually given and fixed. To get around the
dependence on the sample sets, they partition $\Omega$ into 
a predefined number of subsets of equal size.
However, sample sets obtained by partitioning are not
 independent, and partitioning, if used in practice, 
does not make the most efficient use of the data.
Thus, \citet{HarWoo14} considered a new resampling scheme. They assume a known generative
model of $\set{\Omega_k}$, where each $\Omega_k$ is obtained under a Bernoulli model
with probability $p_k$, $p=\sum_k p_k$ and $\Omega = \cup_{k}
\Omega_k$. While not practical, under
this assumption the authors obtain a sample complexity that is logarithmic in $\kappa$.

Another theoretical disadvantage of
the resampling scheme is that the sample complexity depends on the desired accuracy $\epsilon$, as established
by \cite{Keshavan2012, JaiNetSan13, Har14, HarWoo14}. As the accuracy goes to zero,
the sample complexity increases. In contrast, our
algorithm doesn't require resampling, and the sample
complexity is independent of $\epsilon$.

In 2014, \citet{CanLiSol14} proposed \textit{Wirtinger flow} for phase
retrieval. Wirtinger flow is a fast first-order algorithm that
minimizes a fourth order (nonconvex) objective, geometrically
converging to the global optimum.  While previous work
\cite{phaselift_1, phaselift_2, phaselift_3} lifts the phase retrieval
problem into an SDP where the solution is rank one, this work bridges
SDP and first-order algorithms via the Burer-Monteiro technique. It
has inspired further research on related topics; last year, the
authors of \cite{ZheLaf15, TuBocSim16, BhoKyrSan15, CheWai15}
considered factorizations for \eqref{eq:low_rank_linear_inverse},
assuming $\Xstar$ is semidefinite, and proved global optimality of
first-order algorithms under appropriate
initializations. \citet{TuBocSim16} have extended this algorithm to
handle rectangular matrix via asymmetric factorization, and have shown
exact recovery of $\Xstar$, assuming $\A$ satisfies a certain RIP.
They use lifting implicitly, factorizing $X = Z_U Z_V^\T$ and applying
gradient updates on both factors $Z_U$ and $Z_V$ simultaneously, with
the nonconvex objective function
\begin{equation}
    g(Z_U, Z_V) = \frac{1}{2p} \norm{\PP_\Omega (Z_U Z_V^\T - \Xstar)}^2_F + \frac{\lambda'}{4} \norm{Z_U^\T Z_U - Z_V^\T Z_V}^2_F.
\end{equation}
Their proof strategy also shows convergence of $Z$ in the lifted space.
For the specific case of matrix completion, \citet{CheWai15} obtained 
guarantees when $\Xstar$ is semidefinite. Our work generalizes the results obtained in
\cite{TuBocSim16, CheWai15}, extending the recent literature on first-order algorithms
for factorized models.

After completing this work we learned of independent research of
\citet{SunLuo15}, who also analysed a gradient algorithm for rectangular matrix
completion.
Their formulation is similar to ours,
with additional Frobenius norm constraints on the factors. The authors established
a sample complexity of $O(r^7 \kappa^6)$ observations; in comparison our bound
scales as $O(r^2 \kappa^2)$. The authors also analyzed block
coordinate descent type alternating minimization, which cyclically updates the
rows of $U$ and then the rows of $V$, showing exact recovery of this
algorithm without resampling. Recent independent work of \citet{YiParChe16}
analyzes a gradient scheme for Robust PCA. Under the setting of partial
observation without corruption, this is the standard matrix completion problem.
In other related work, \cite{ZhaWanLiu15, WeiCaiCha16} also study
nonconvex optimization methods for matrix completion, using 
algorithms that still require low rank SVD in each iteration.

%% file: expr.tex
\section{Experiments}
\label{sec:expr}
We conduct experiments on synthetic datasets to support our
analytical results. 
As the column space regularizer and incoherence constraint of
our gradient method (\texttt{GD}) are merely for analytical
purpose, we drop them in all the experiments; simply
optimize the $\ell_2$ loss $\frac{1}{2} \norm{\PP_{\OmegaY}(ZZ^\T - \Ystar)}^2_F$. 
We compare \texttt{GD} with \texttt{SVP}, \texttt{OptSpace},
nuclear norm minimization (\texttt{nuclear})
and trust region methods on Riemannian manifolds (\texttt{trustRegion}).
For \texttt{nuclear}, we rescale the standard objective to be
\begin{equation}
    \label{eq:nuclear}
    \min_X \quad \frac{1}{2\lambda} \norm{\PP_\Omega(X - \Xstar)}^2_F + \norm{X}_*,
\end{equation}
where $\lambda = 0$ will enforce the minimizer fitting the observed
values exactly. We use ADMM to solve \eqref{eq:nuclear}. It is based on the algorithm for
\textit{the matrix approach} in \cite{TomHayKas10}, and can neatly handle the case
$\lambda = 0$. We emphasize there is no computational difference between cases whether
$\lambda$ is zero or not. All methods are implemented in MATLAB. We use the toolbox \texttt{Manopt} for
\texttt{trustRegion} \cite{BouMis14} and the implementation of \texttt{OptSpace}
from the authors. For \texttt{AltMin}, we use the same sample sets in every
iteration. The experiments were run on a Linux machine with a 3.4GHz Intel Core i7 processor and 8 GB memory.

\paragraph{Computational Complexity} 
Table~\ref{tbl:periter_complexity} summarizes the per-iteration complexity of all the methods for completing a $n
\times n$ matrix.
 Since $M^k$ is a sparse matrix with $m$ nonzero entries, and we have dropped the
 regularizer and constraint,
our method \texttt{GD} only needs $2mr + m + n^2 r$ operations to compute the
gradient, and $4nr$ operations to update the iterate.
The computation of \texttt{nuclear} is dominated by
singular value thresholding and updating the objective value, which require the
$O(n^3)$ cost full SVD. Similarly, \texttt{SVP} needs $O(n^2 r)$ operations to
compute the rank-$r$ SVD for low rank projection. 
For \texttt{OptSpace}, $O(mr + n^2 r + n r^2)$
operations are needed to compute the manifold gradient and line search. 
The most expensive part is to determine the optimal scaling matrix $S \in \R^{r \times
r}$, which boils down to solving a $r^2$ by $r^2$ dense linear system. 
In total $O(mr^3 + n^2r^2 + nr^4 + r^6)$ operations are used to construct and solve this
system.
For \texttt{AltMin}, in every iteration we have to solve $(n_1 + n_2)$ linear
systems of size $r \times r$. See \cite{Sun15} for the exact formulation. The time cost for each
iteration is $O(mr^2)$.
One can see that \texttt{GD} reduces the computation than the others. 
Though the dominating terms for \texttt{SVP} and
\texttt{GD} are in the same order, in practice the partial SVD are more expensive than the gradient update, especially
on large instances.

\begin{table}[htb]
    \centering
    \begin{tabular}{c l c}
        \toprule
        Method & Complexity\\ 
        \midrule
        \texttt{GD} & $2mr + m + n^2 r + 4nr$ \\
        \texttt{SVP} & $O(n^2 r)$\\
        \texttt{OptSpace} & $O(mr^3 + n^2 r^2 +  nr^4 +  r^6)$\\
        \texttt{nuclear}&  $O(n^3)$\\
        \texttt{AltMin}&  $O(mr^2)$\\
        \bottomrule
    \end{tabular}
    \vskip5pt
    \caption{Per-iteration complexities.}
    \label{tbl:periter_complexity}
\vskip-15pt
\end{table}
\paragraph{Runtime Comparison}
We randomly generated a true matrix $\Xstar$ of size $4000 \times 2000$ and rank
3. It is constructed from the rank-$3$ SVD of a random $4000 \times 2000$ matrix
with i.i.d normal entries. We sampled $m = 199057$ entries of $\Xstar$ uniformly at
random, where $m$ is roughly equal to $2 n r \log n$ with $n = 4000$ and $r =
3$. For simplicity, we feed \texttt{SVP},
\texttt{OptSpace} and \texttt{GD} with the true rank.  For all these methods, we
use the randomized algorithm of \citet{HalMarTro11} to compute the low rank SVD,
which is approximately 15 times faster than MATLAB built-in SVD on instances of
such size. We report
relative error measured in the Frobenius norm, defined as $\|\hat{X} -
    \Xstar \|_F / \norm{\Xstar}_F$.  For \texttt{nuclear}, we set $\lambda = 0$ to
enforce exact fitting. The convergence speed of ADMM mildly depends on the
choice of penalty parameter.  We tested $5$ values $0.1, 0.2, 0.5,
1, 1.5$ and selected $0.2$, which leads to fastest convergence. Similarly, for \texttt{SVP}, we would
like to choose the largest step size for which the algorithm is converging. We
evaluated $15, 20, 30, 35, 40$ and selected $30$. The step size is
chosen for \texttt{GD} in the same way. Five values $20, 50, 70, 75, 80$ are
tested for $\eta$ and we picked $70$. For \texttt{OptSpace}, we
compared fixed step sizes $0.5 0.1 0.05 0.01 0.005$ with line search,
and found the algorithm converged fastest under line search.
Figure~\ref{fig:runtime} shows the results. \texttt{GD} is slightly slower than
\texttt{trustRegion} and faster than competing approaches.

To further illustrate how runtime scales as the dimension increases,
we run larger instances of size $10000 \times
5000$ and $20000 \times 5000$, where the true rank is $40$. The parameters are selected in
the same manner, and we terminate the computation once the relative error is below
$1e^{-9}$. We report the results of \texttt{AltMin} \texttt{GD}, \texttt{SVP} and \texttt{trustRegion} in
Figure~\ref{fig:runtime_large}; \texttt{nuclear}, \texttt{OptSpace} do not scale well to such
sizes so that we didn't include them. The runtime of
\texttt{AltMin} scales the slowest, while the runtimes
of \texttt{GD} and \texttt{trustRegion} increase slower than \texttt{SVP}.

\paragraph{Sample Complexity}
We evaluate the number of observations required by \texttt{GD} for exact
recovery.
For simplicity, we consider square but asymmetric $\Xstar$.
We conducted experiments in 4 cases, where the randomly
generated $\Xstar$ is of
size $500 \times 500$ or $1000 \times 1000$, and of rank $10$ or $20$.
In each case, we compute the solutions of \texttt{GD} given $m$ random
observations, and a solution with relative error below $1e^{-6}$ is considered to
be successful. We run 20 trials and compute the empirical probability of successful recovery.
The results are shown in Figure~\ref{fig:gd_sample}. 
For all four cases, the phase transitions occur around $m \approx 3.5 n r$. 
This suggests that the actual sample complexity of \texttt{GD} may scale linearly with both the dimension $n$
and the rank $r$.

\begin{figure}[tb]
    \centering
    \subfloat[] {
        \hskip-20pt
        \includegraphics[width=0.45\textwidth]{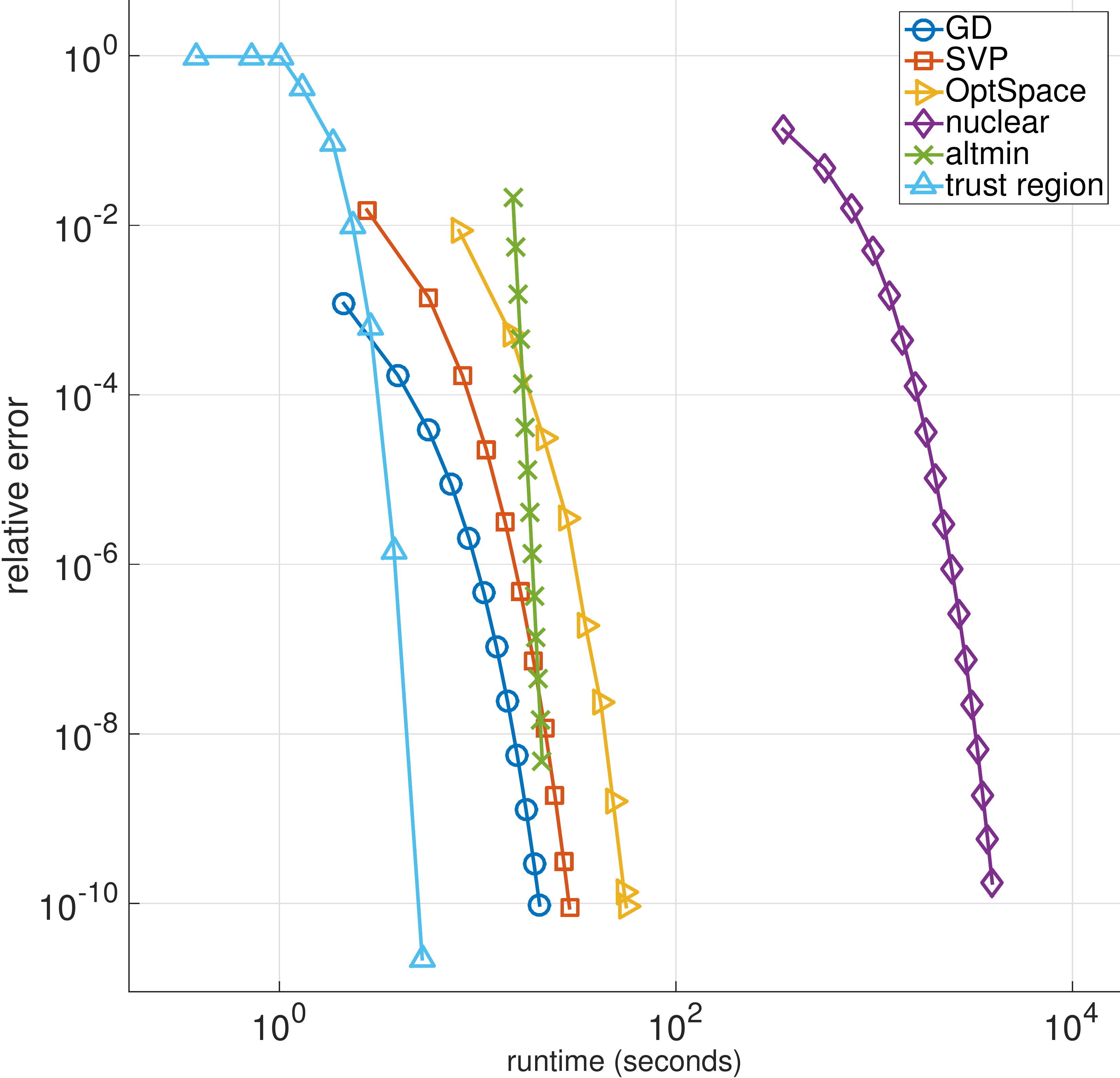}
        \label{fig:runtime}
    }
    \subfloat[] {
        \includegraphics[width=0.45\textwidth]{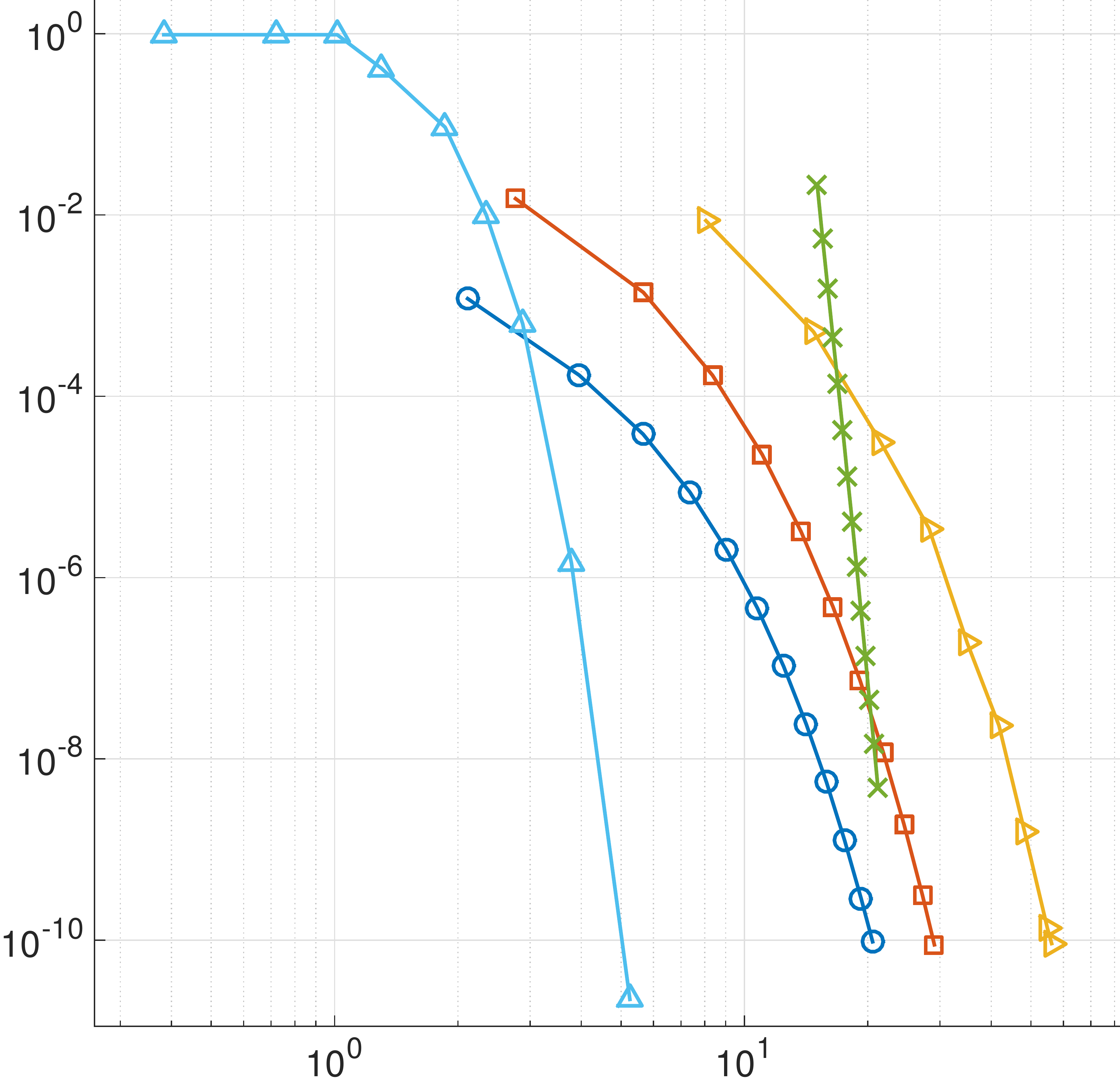}
        \label{fig:runtime_zoom}
    }
    \caption{(a) Runtime comparison where $\Xstar$ is $4000 \times 2000$ and
    of rank $3$. $199057$ entries are observed. (b) Magnified plots to compare
    other methods except \texttt{nuclear}.}
\end{figure}
\begin{figure}[tb]
    \subfloat[] {
        \hskip4pt
        \includegraphics[width=0.45\textwidth]{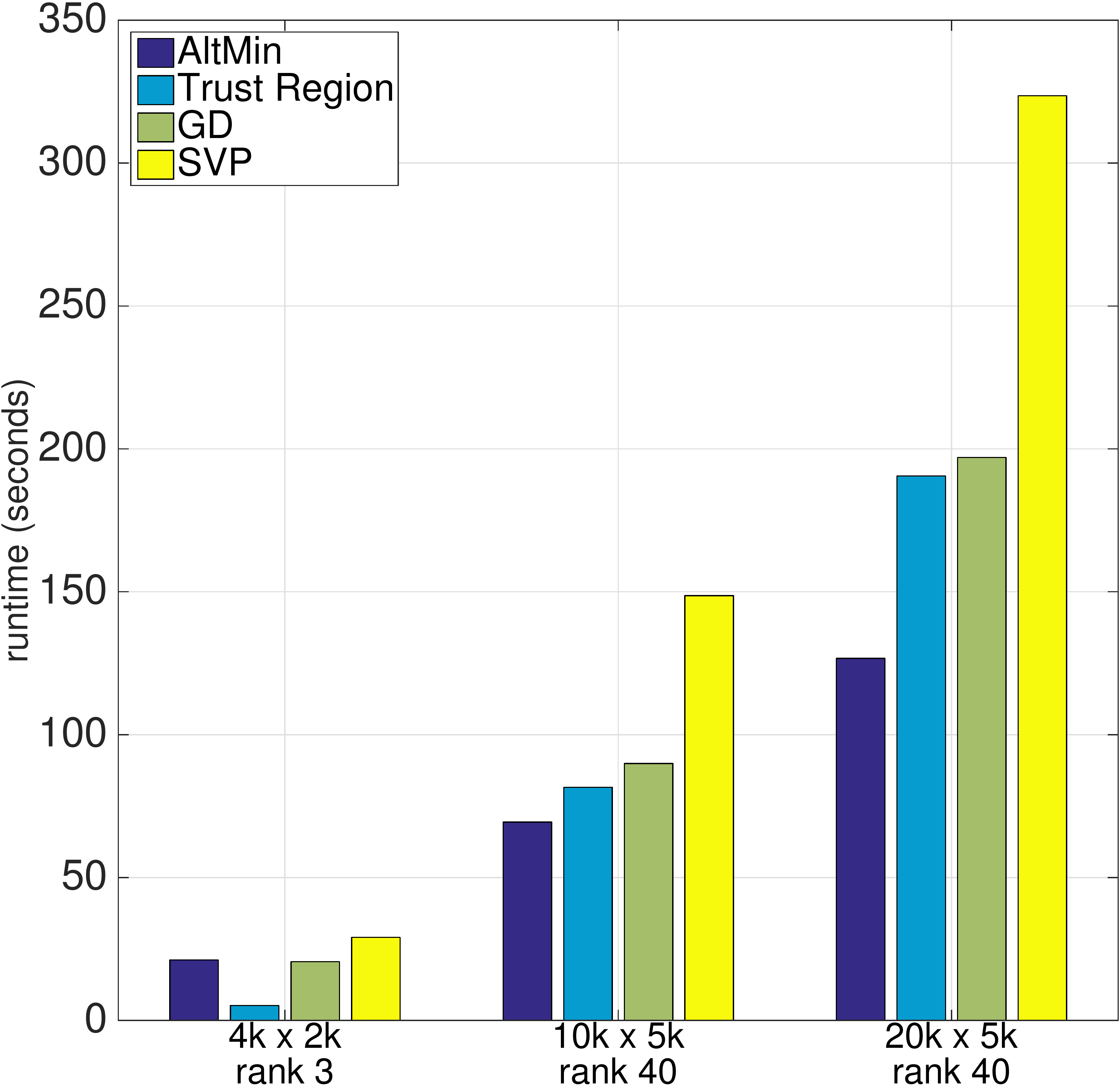}
        \label{fig:runtime_large}
    }
    \subfloat[]{
        \hskip4pt
        \includegraphics[width=0.45\textwidth]{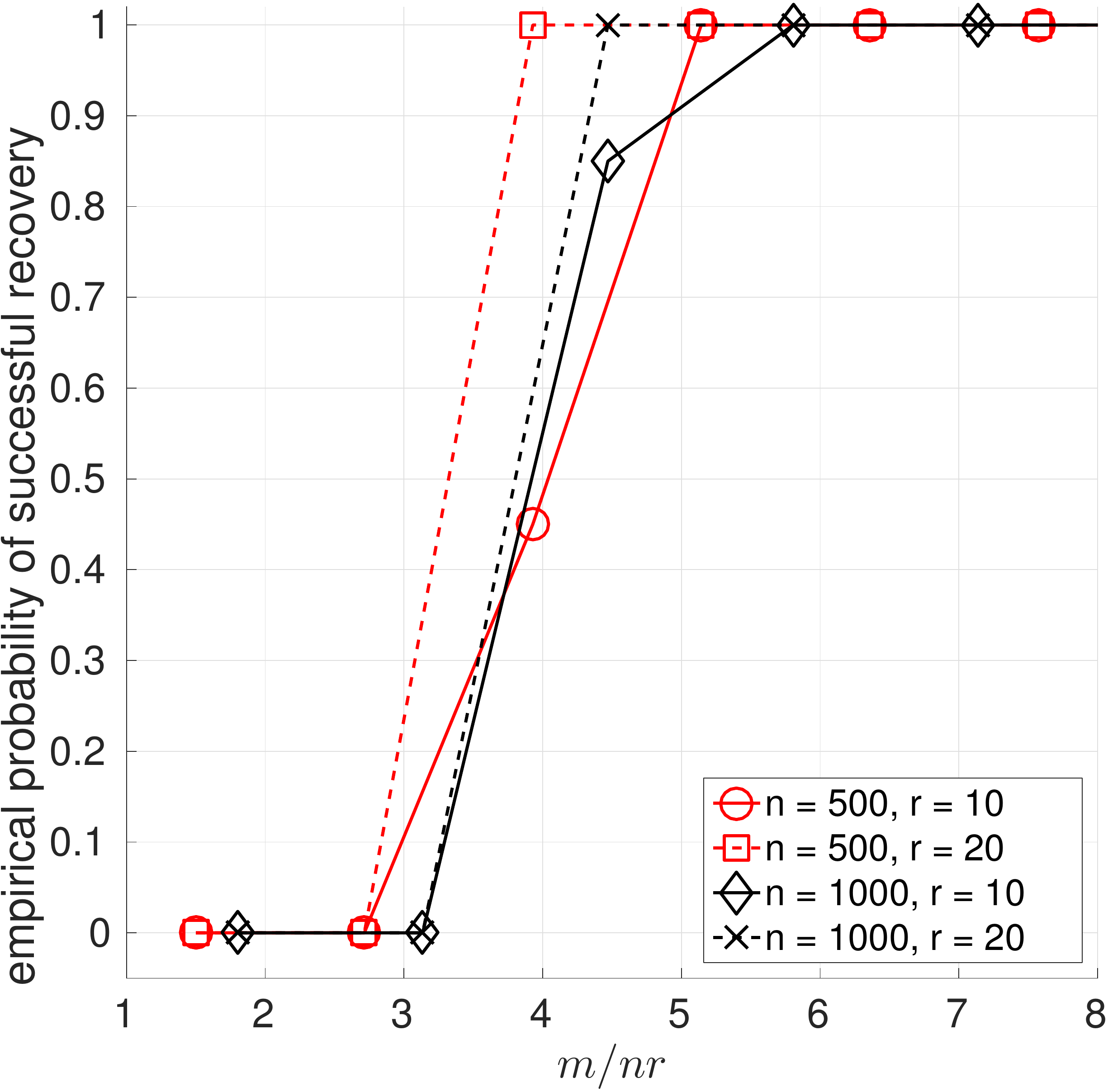}
        \label{fig:gd_sample}
    }
    \caption{ (a) Runtime growth of \texttt{AltMin}, \texttt{trustRegresion}, \texttt{GD} and \texttt{SVP}. (b) Sample complexity of gradient scheme. }
\end{figure}

%% file: conclude.tex
\section{Conclusion}
We propose a lifting procedure together with Burer-Monteiro
factorization and a first-order algorithm to carry out rectangular
matrix completion.  While optimizing a nonconvex objective, we
establish linear convergence of our method to the global optimum with
$O(\mu r^2 \kappa^2 n \max(\mu, \log n))$ random observations. We conjecture that
$O(n r)$ observations are sufficient for exact recovery, and that the
column space regularizer can be dropped. We provide empirical evidence
showing this simple algorithm is fast and scalable, suggesting that lifting techniques
may be promising for much more general classes of problems.
\label{sec:conclude}

%% file: ack.tex
\section*{Acknowledgements}
The authors thank Rina Foygel Barber, Yudong Chen and Ruoyu Sun for helpful
comments. Research supported in part by ONR grant N00014-15-1-2379 and NSF grant DMS-1513594.

%% file: app_ingred.tex
\section{Technical Lemmas}
Another way of writing the objective function is
\[
    f(Z) = \frac{1}{2p} \sum_{l = 1}^{2m} \left( \ip{A_l, Z Z^\T} - b_l \right)^2 + \frac{ \lambda }{4} \norm{Z^\T D Z}^2_F,
\]
where $l$ is an index of $\OmegaY$, $A_l$ is a
matrix with $1$ at the corresponding observed entry and $0$ elsewhere.
Let $H = Z - \Zbar$, the gradient can be written as
\[
    \begin{aligned}
        \nabla f(Z) & = \frac{1}{p} \sum_{l = 1}^{2m} \left( \ip{A_l, Z Z^\T} - b_l \right) ( A_l + A^\T_l ) Z
            + \lambda \overbrace{ D Z  \left( Z^\T D Z \right) }^\Gamma \\
            & =  \frac{1}{p} \sum_{l = 1}^{2m} \left( \ip{A_l, H {\Zbar}^\T + \Zbar H^\T + H H^\T } \right) ( A_l + A^\T_l ) (\Zbar + H)
            + \lambda \Gamma. \\
    \end{aligned}
\]

We will use the following facts throughout the proof:
\begin{align}
    \label{eq:Zbar_2_infty}
    & \norm{\Zbar}_{2, \infty} = \norm{\Zstar}_{2,\infty} \leq  \sqrt{ \frac{\mu
            r}{\nmin} \sstar_1 },\\
    \label{eq:H_2_infty}
    & \norm{H}_{2, \infty} \leq  3 \sqrt{ \frac{\mu r}{\nmin} \sstar_1 }, \\
    \label{eq:trace_symm}
    & \ip{(A_l + A^\T_l)B, C} = \ip{A_l, BC^\T + CB^\T}, \\
    \label{eq:H'Z_symm}
    & Z^\T \Zbar \; \text{is positive semidefinite,} \; H^\T \Zbar \; \text{is symmetric.}
\end{align}

Inequality~\eqref{eq:Zbar_2_infty} is a direct result of Definition ~\ref{defn:incoh}.
To see \eqref{eq:H_2_infty}, note that
$ \norm{H}_{2, \infty} \leq  \norm{Z}_{2, \infty} + \norm{\Zbar}_{2, \infty}
\leq  \sqrt{ \frac{2 \mu r}{\nmin} \sigma_1 } + \sqrt{ \frac{\mu r}{\nmin} \sstar_1 }$,
and $| \sigma_1 - \sstar_1 | \leq \frac{1}{16}\sstar_1$ by the discussion of
initialization in Appendix~\ref{sec:proof_init}.
For \eqref{eq:H'Z_symm}, it holds that
\[ \argmin_{RR^\T = R^\T R = I} \norm{Z - \Zstar R}^2_F = AB^\T, \]
 where $A \Lambda B^\T$ is the SVD of ${\Zstar}^\T Z$.
Clearly, $Z^\T \Zbar$ is positive semidefinite,  and $H^\T \Zbar = Z^\T \Zbar - {\Zbar}^\T \Zbar = B \Lambda B^\T -
{\Zbar}^\T \Zbar$ is symmetric.

Next, we list several technical lemmas that are utilized later.
We will use $c$ to denote a numerical constant, whose value may vary from line to line.
\begin{lemma}
    \label{lem:ZZ_to_X}
    For any $Z$ of the form
    $Z  =  \begin{bmatrix} Z_U \\ Z_V \end{bmatrix} = \begin{bmatrix} U \Shalf R \\ V \Shalf R \end{bmatrix}$,
where $U, V, R$ are unitary matrices and $\Sigma \succeq 0$ is a diagonal matrix, we
have
    \[    \norm{ZZ^\T - \Zstar {\Zstar}^\T}_F \leq 2 \norm{U\Sigma V^\T - \Ustar \Sstar {\Vstar}^\T}_F. \]
\end{lemma}
\begin{proof}
Recall that
\[
    \Zstar  =  \begin{bmatrix} \Zstar_U \\ \Zstar_V \end{bmatrix} =
    \begin{bmatrix} \Ustar
    \Sstarhalf \\ \Vstar \Sstarhalf \end{bmatrix}
\]
where $\Xstar = \Ustar \Sstar {\Vstar}^\T$.
We have
\begin{equation}
    \begin{aligned}
    \label{eq:ZZ_to_X_helper0}
     & && \hskip-30pt \norm{ZZ^\T - \Zstar {\Zstar}^\T}^2_F\\
     & = && \norm{U\Sigma U^\T - \Ustar \Sstar {\Ustar}^\T }^2_F +
    \norm{V\Sigma V^\T - \Vstar \Sstar {\Vstar}^\T }^2_F +
    2 \norm{U\Sigma V^\T - \Ustar \Sstar {\Vstar}^\T }^2_F,
    \end{aligned}
\end{equation}
and
\begin{equation}
    \label{eq:ZZ_to_X_helper1}
    \begin{aligned}
        &    && \hskip-70pt \norm{U\Sigma U^\T - \Ustar \Sstar {\Ustar}^\T }^2_F + \norm{V\Sigma
    V^\T - \Vstar \Sstar {\Vstar}^\T }^2_F  \\
    & = &&2\left( \norm{\Sigma}^2_F + \norm{\Sstar}^2_F -  \ip{\Sigma, U^\T
    {\Ustar}^\T \Sstar {\Ustar}^\T U  + V^\T {\Vstar}^\T \Sstar {\Vstar}^\T V } \right)\\
    \end{aligned}.
\end{equation}
We can obtain the lower bound
\begin{equation}
    \label{eq:ZZ_to_X_helper2} \begin{aligned}
        &  && \hskip-1in \ip{\Sigma, \,  U^\T {\Ustar}^\T \Sstar {\Ustar}^\T U 
   + V^\T {\Vstar}^\T \Sstar {\Vstar}^\T V } \\
   &  = && \sum_{i=1}^r \sigma_i \left ( U^\T {\Ustar}^\T \Sstar {\Ustar}^\T U 
   + V^\T {\Vstar}^\T \Sstar {\Vstar}^\T V \right)_{ii}  \\ 
   &  = && \sum_{i=1}^r \sigma_i \sum_{k=1}^r \sstar_k  \left(  (U^\T \Ustar)^2_{ik}
   + (V^\T \Vstar)^2_{ik} \right) \\
   &  \geq && \sum_{i=1}^r \sigma_i \sum_{k=1}^r \sstar_k  \cdot 2 \, (U^\T \Ustar)_{ik}
    (V^\T \Vstar)_{ik}  \\
    &  = && 2 \ip{ \Sigma,  \, U^\T \Ustar \Sstar {\Vstar}^\T V }.
    \end{aligned}
\end{equation}
Combining \eqref{eq:ZZ_to_X_helper1} and \eqref{eq:ZZ_to_X_helper2}, we obtain
\begin{equation}
    \label{eq:ZZ_to_X_helper3}
    \begin{aligned}
        & &&  \hskip-1in \norm{U\Sigma U^\T - \Ustar \Sstar {\Ustar}^\T }^2_F + \norm{V\Sigma
    V^\T - \Vstar \Sstar {\Vstar}^\T }^2_F  \\
        & \leq &&  2\left( \norm{\Sigma}^2_F + \norm{\Sstar}^2_F -  2\ip{\Sigma, U^\T
    {\Ustar}^\T \Sstar {\Vstar}^\T V } \right)\\
        & = &&  2\left( \norm{U \Sigma V^\T }^2_F + \norm{\Ustar \Sstar
{\Vstar}^\T }^2_F -  2\ip{U\Sigma V^\T, {\Ustar}^\T \Sstar {\Vstar}^\T } \right)\\
    & = && 2 \norm{U \Sigma V^\T - \Ustar \Sstar {\Vstar}^\T }^2_F.
    \end{aligned}
\end{equation}
Plugging \eqref{eq:ZZ_to_X_helper3} back into \eqref{eq:ZZ_to_X_helper0}, we obtain the lemma.
\end{proof}

Recall that $n = \max(n_1, n_2)$.
We will exploit the following two known concentration results.
\begin{lemma}[\citet{Che15}, Lemma 2] 
    \label{lem:Chen_opnorm}
    For any fixed matrix $\Xstar \in \R^{n_1 \times n_2}$,
    there exist universal constants $c, c_1, c_2$ such that
    with probability at least $1 - c_1 n^{-c_2}$,
    \[  \norm{ p^{-1} \PP_{\Omega}(\Xstar) - {\Xstar} } \leq c \left( \frac{\log n}{p} \norm{\Xstar}_\infty 
                            + \sqrt{\frac{\log n}{p}} \norm{\Xstar}_{\infty, 2} \right).     \]
\end{lemma}

\begin{lemma}[\citet{CanRec09}, Theorem 4.1]
    \label{lem:Rudelson}
Define subspace
\begin{equation}
    \label{eq:defn_T}
T = \set{M \in \R^{n_1 \times n_2} : M = \Ustar X^\T + Y {\Vstar}^\T \; \text{for some} \; X \; \text{and} \; Y}.
\end{equation}
Let $\PP_T$ be the Euclidean projection onto $T$.
There is a numerical constant $c$ such that for any $\delta \in (0, 1]$, if $p
\geq \dfrac{c}{\delta^2} \dfrac{\mu r \log n }{\nmin}$, then with probability $1 -
3n^{-3}$, we have
\[
    p^{-1}\norm{ \PP_T\PP_\Omega\PP_T - p \PP_T } \leq \delta.
\]
\end{lemma}

Lemma~\ref{lem:op_norm_random_graph} upper bounds the spectral
norm of the adjacency matrix of a random 
Erd\H{o}s-R\'enyi graph. It is a variant of Lemma
7.1 of \citet{KesMonOh09}, which uses known results of \citet{FeiOfe05}. 
\begin{lemma}[\citet{CheWai15}, Lemma 9]
    \label{lem:op_norm_random_graph}
    Suppose that $\bar{\Omega} \subset [d] \times [d]$ is the set of edges of a
    random Erd\H{o}s-R\'enyi graph with $n$ nodes, where any pair of nodes 
    is connected with probability $p$. There exists two numerical constants $c_1, c_2$ such that, 
    for any $\delta \in (0, 1]$, if $p \geq \dfrac{c_1 \log d}{\delta^2 d}$, 
    then with probability at least $1 - \frac{1}{2} d^{-4} $, uniformly for all
    $x, y \in \R^n$ it holds that
    \begin{equation}
        \label{eq:op_norm_random_graph}
      p^{-1} \sum_{(i,j) \in \bar{\Omega}} x_i y_j \leq (1+\delta) \norm{x}_1 \norm{y}_1 + c_2 \sqrt{ \frac{d}{p} } \norm{x}_2 \norm{y}_2.
    \end{equation}
\end{lemma}
We refer readers to \cite{KesMonOh09} for a complete proof, in particular
noticing that one can choose $p$ large enough so that the constant factor in the
first term in \eqref{eq:op_norm_random_graph} is only $1 + \delta$.

Lemma~\ref{lem:PP_HH}, \ref{lem:upper_PP_Y_AB'} and \ref{lem:rsc_chen} are direct
generalizations of Lemma 4 and 5 of \cite{CheWai15}.

\begin{lemma} There exists a constant $c$ such that, for any $\delta \in (0,
1]$, if $p \geq \frac{c}{\delta^2} \max\left(\frac{\log (n_1 + n_2)}{n_1 + n_2}, \frac{\mu^2 r^2
    \kappa^2}{\nmin}\right)$, then with probability at least $1 - \frac{1}{2} (n_1+n_2)^{-4}$,
    uniformly for all $H$ such that $\norm{H}_{2, \infty} \leq 3 \sqrt{
    \frac{\mu r }{\nmin}\sstar_1 }$, we have
\[
     p^{-1} \norm{\PP_{\OmegaY}(HH^\T)}^2_F \leq (1 + \delta)  \norm{H}^4_F + \delta \sstar_r \norm{H}^2_F.
\]
\label{lem:PP_HH}
\end{lemma}
\begin{proof}
    It holds that
    \begin{equation}
        \begin{aligned} 
            p^{-1} \norm{\PP_{\OmegaY}(HH^\T)}^2_F & = && p^{-1} \sum_{(i,j) \in \OmegaY} \ip{H_{(i)}, H_{(j)}}^2 \\
            & \leq && p^{-1} \sum_{(i,j) \in \OmegaY} \norm{H_{(i)}}^2_2 \norm{H_{(j)}}^2_2. \\
        \end{aligned} 
    \end{equation}
    Since $\OmegaY$ is a reduced sampling of $Y \in \R^{(n_1 + n_2) \times (n_1 + n_2)}$ under
    a Bernoulli model, Lemma~\ref{lem:op_norm_random_graph} is applicable here.
    Assume $p \geq \frac{c_1 \log (n_1 + n_2) }{\delta^2 (n_1 + n_2)}$, we then have with
    probability at least $1 - \frac{1}{2}(n_1+n_2)^{-4}$,  for all $H$ such that
    $\norm{H}_{2, \infty} \leq 3 \sqrt{ \frac{\mu r}{\nmin}\sstar_1 }$, 
    \begin{equation}
        \begin{aligned} 
            p^{-1} \norm{\PP_{\OmegaY}(HH^\T)}^2_F
            & \leq && p^{-1} \sum_{(i,j) \in \OmegaY} \norm{H_{(i)}}^2_2 \norm{H_{(j)}}^2_2 \\
            & \stackrel{(a)}{\leq} && (1+\delta)  \Big( \sum_{ i \in [n_1+n_2] }\norm{H_{(i)}}^2_2
                \Big)^2 + c_2 \sqrt{\frac{n_1+n_2}{p}} \sum_{i \in [n_1+n_2] } \norm{H_{(i)}}^4_2\\
            & \leq && (1+\delta) \norm{H}^4_F + c_2 \sqrt{\frac{n_1+n_2}{p}}
            \norm{H}^2_F \norm{H}^2_{2, \infty} \\
            & \stackrel{(b)}{\leq} &&  \norm{H}^2_F \Big( (1+\delta) \norm{H}^2_F + \sqrt{
            \frac{ 81 c^2_2 \mu^2 r^2 {\sstar_1}^2 (n_1 + n_2) }{p (\nmin)^2 }
            }\Big),
        \end{aligned}
    \end{equation}
    where $(a)$ follows from Lemma~\ref{lem:op_norm_random_graph} and
    $(b)$ follows from $\norm{H}_{2, \infty} \leq 3 \sqrt{ \frac{\mu
    r}{\nmin}\sstar_1 }$.
    
    Let us further assume $p \geq \frac{162 c^2_2 \mu^2 r^2 \kappa^2 \gamma
    }{\delta^2 (\nmin)}$, where $\gamma = n / (\nmin)$ is a fixed constant,
    then we can bound
    \begin{equation}
        \begin{aligned} 
            p^{-1} \norm{\PP_{\OmegaY}(HH^\T)}^2_F 
            & \leq && \norm{H}^2_F  \left( (1+\delta)  \norm{H}^2_F + \delta
                \sstar_r \right) .
        \end{aligned} 
    \end{equation}
    The final threshold we obtain is thus $ p \geq \frac{c}{\delta^2} \max\left(\frac{
    \log (n_1 + n_2)}{n_1 + n_2}, \frac{\mu^2 r^2 \kappa^2}{\nmin} \right)$ for some constant $c$.
\end{proof}

\begin{lemma} 
\label{lem:upper_PP_Y_AB'}
There exists a constant $c$, if $p \geq \dfrac{c \log n}{\nmin}$, 
then with probability at least $1 - 2n_1^{-4} - 2n_2^{-4}$, uniformly for all
matrices $A$, $B$ such that $AB^\T$ is of size $(n_1 + n_2)
\times (n_1 + n_2)$,
\[
p^{-1} \norm{\PP_{\OmegaY}(AB^\T)}^2_F \leq 2n \min \Big\{ \norm{A}^2_F \norm{B}^2_{2,\infty}, \norm{B}^2_F \norm{A}^2_{2, \infty} \Big\}
\]
\end{lemma}
\begin{proof}
    Let $\Omega_{Y_i} = \set{j: (i,j) \in \OmegaY}$ denote the set of entries sampled in the $i$th row of
    $AB^\T$. Note that because of the structure of $\OmegaY$, at most $n_2$
    entries are sampled at the frist $n_1$ rows, and at most $n_1$ entries are
    sampled at the rest $n_2$ rows.

    Using a binomial tail bound, if $p \geq \dfrac{c \log n_2}{n_2}$ for
    sufficiently large $c$, the event $\max_{i \in [n_1]} |\Omega_{Y_i}|
    \leq 2pn_2$ holds with probability at least $1 - n_2^{-4}$. Similarly for
    the rest $n_2$ rows. Hence, if $p \geq \dfrac{c \log n}{\nmin}$ for some
    constant $c$, with probability at least $1 - n_1^{-4} - n_2^{-4}$, we have
    $\max_{i \in [n_1 + n_2]} |\Omega_{Y_i}| \leq 2pn$.
    
    Conditioning on this event, we then have for all $A, B$ of proper size,
    \[
        \begin{aligned}
            p^{-1} \norm{\PP_{\OmegaY}(AB^\T)}^2_F
            & =    && p^{-1} \sum_{i=1}^{n_1 + n_2} \sum_{j \in \Omega_{Y_i}} \ip{A_{(i)}, B_{(j)}}^2 \\
            & \leq && p^{-1} \sum_{i=1}^{n_1 + n_2} \norm{A_{(i)}}^2_2 \sum_{j \in \Omega_{Y_i}} \norm{B_{(j)}}^2_2 \\
            & \leq && p^{-1} \sum_{i=1}^{n_1 + n_2} \norm{A_{(i)}}^2_2 \max_{i \in [n_1+n_2]} | \Omega_{Y_i} | \norm{B}^2_{2, \infty}\\
            & \leq && 2n \norm{A}^2_F \norm{B}^2_{2, \infty}.
        \end{aligned}
    \]
    Similarly we can prove with probability at least $1 - n_1^{-4} - n_2^{-4}$,
    \[
        p^{-1} \norm{\PP_{\OmegaY}(AB^\T)}^2_F \leq 2n \norm{B}^2_F \norm{A}^2_{2, \infty}.
    \]
\end{proof}

The following lemma establishes restricted strong convexity and smoothness of
the observation operator for matrices in $T$.
\begin{lemma}
\label{lem:rsc_chen}
Let $T$ be the subspace defined in \eqref{eq:defn_T}.
There exists a universal constant $c$ such that, if $p \geq \dfrac{ c
}{\delta^2} \dfrac{ \mu r \log n}{ \nmin }$, with probability at least $1 - 3n^{-3}$,
uniformly for all $A \in T$, we have
\begin{equation}
    \label{eq:rsc_chen_1}
    p(1-\delta)\norm{A}^2_F \leq \norm{\PP_\Omega(A)}^2_F \leq p(1+\delta)\norm{A}^2_F.
\end{equation}
Consequently, uniformly for all $A, B \in T$,
\begin{equation}
    \label{eq:rsc_chen_2}
    \abs{p^{-1} \ip{\PP_{\Omega}(A),  \PP_{\Omega}(B)} -  \ip{A,  B} }  \leq
    \delta \norm{A}_F \norm{B}_F.
\end{equation}
\end{lemma}
\begin{proof}
    By Lemma~\ref{lem:Rudelson}, with probability at least $1 - 3 n^{-3}$,
    for any $X \in \R^{n_1 \times n_2}$ it holds that
    \begin{equation}
        \label{eq:rsc_chen_helper}
        p (1-\delta) \norm{X}_F \leq \norm{\PP_T \PP_\Omega \PP_T(X) }_F \leq
        p(1+\delta) \norm{X}_F.
    \end{equation}
   Let $A$ be a matrix in $T$.
   Rewriting $\norm{\PP_\Omega(A)}^2_F =
   \ip{\PP_\Omega \PP_T (A), \PP_\Omega \PP_T (A ) } = \ip{A, \PP_T
     \PP_\Omega \PP_T (A ) }$, and using the Cauchy-Schwarz
   inequality and \eqref{eq:rsc_chen_helper} we can bound
   \begin{equation}
        \label{eq:rsc_chen_helper1}
         \norm{\PP_\Omega(A)}^2_F \leq p(1+\delta) \norm{A}^2_F.
   \end{equation}
   In addition, we have
   \begin{equation}
       \label{eq:rsc_chen_helper2}
    \begin{aligned}
        \norm{\PP_\Omega(A)}^2_F  & = \ip{A, \PP_T \PP_\Omega \PP_T (A ) }\\
                    & = \ip{A, \PP_T \PP_\Omega \PP_T (A ) - p \PP_T(A) + p \PP_T(A) }\\
                    & \geq -\norm{A}_F \norm{ (\PP_T \PP_\Omega \PP_T - p \PP_T)(A)}_F + p \norm{A}^2_F \\
                    & \stackrel{(a)}{\geq} p(1-\delta) \norm{A}^2_F,
    \end{aligned}
   \end{equation}
    where $(a)$ follows from Lemma~\ref{lem:Rudelson}. Combining
    \eqref{eq:rsc_chen_helper1} and \eqref{eq:rsc_chen_helper2} proves~\eqref{eq:rsc_chen_1}.
    To show \eqref{eq:rsc_chen_2}, let $A' = \frac{A}{\norm{A}_F}$ and $B' =
    \frac{B}{\norm{B}_F}$. Both $A' + B'$ and
    $A' - B'$ are in $T$. We have
    \begin{equation}
        \begin{aligned}
        \ip{\PP_\Omega(A'), \PP_\Omega(B')} & = \frac{1}{4} \bigg\{ \overbrace{
            \norm{\PP_\Omega(A' + B')}^2_F }^{\circled{1}} -
        \overbrace{\norm{\PP_\Omega(A' - B')}^2_F}^{\circled{2}} \bigg\} \\
        & \stackrel{(b)}{\leq} \frac{1}{4} \bigg\{ (1+\delta)p\norm{A'+B'}^2_F -
        (1-\delta)p\norm{A' - B'}^2_F \bigg\} \\
        & = \frac{1}{4} \bigg\{ 2\delta p \left( \norm{A'}^2_F +
            \norm{B'}^2_F\right) + 4p \ip{A', B'}    \bigg\} \\
        & = p \delta + p \ip{A', B'},
        \end{aligned}
    \end{equation}
    where $(b)$ follows from \eqref{eq:rsc_chen_1}. Thus, we have
    \begin{equation}
        p^{-1} \ip{\PP_\Omega(A), \PP_\Omega(B)} = p^{-1} \norm{A}_F \norm{B}_F
        \ip{\PP_\Omega(A'), \PP_\Omega(B')} \leq
         \delta \norm{A}_F\norm{B}_F + \ip{A, B}.
    \end{equation}
    Similarly, we can show
    \begin{equation}
        p^{-1} \ip{\PP_\Omega(A), \PP_\Omega(B)} \geq - \delta \norm{A}_F\norm{B}_F + \ip{A, B}.
    \end{equation}

\end{proof}

Last, we want to show the projection onto feasible set $\C$ is a contraction.
\begin{lemma}
    \label{lem:contraction}
    Let $y \in \R^r$ be a vector such that $\norm{y}_2 \leq \theta$, for any $x
    \in \R^r$. Then
    \[ \norm{ \PP_{\norm{\cdot}_2 \leq \theta} (x) - y }^2_2 \leq \norm{x - y}^2_2. \]
\end{lemma}
\begin{proof}
    If $\norm{x}_2 \leq \theta$, then $  \PP_{\norm{\cdot}_2 \leq \theta} (x) = x$. 
    Otherwise
    $\PP_{\norm{\cdot}_2 \leq \theta} (x) = \theta \xbar$,  where $\xbar = \frac{x}{\norm{x}_2}$.
    Write $y =  ( y^\T \xbar ) \xbar + \PP^\perp_{x}(y)$, we have
    \begin{equation}
        \norm{ \theta \xbar - y }^2_2 = \norm{ \theta \xbar -  (y^\T \xbar) \xbar
        }^2_2 + \norm{\PP^\perp_x (y)}^2_2 = (\theta - y^\T \xbar )^2 +
        \norm{\PP^\perp_x (y)}^2_2.
    \end{equation}
    It suffices to show 
    \begin{equation}
        \label{eq:contraction_helper}
        (\theta - y^\T \xbar )^2 \leq  ( \norm{x} - y^\T \xbar)^2. 
    \end{equation}
    If $y^\T \xbar \leq 0$, then \eqref{eq:contraction_helper} holds because $\norm{x} > \theta$.
    If $y^\T \xbar > 0$, \eqref{eq:contraction_helper} still holds since
    $\norm{x} > \theta \geq \norm{y} \geq y^\T \xbar$.
\end{proof}

%% file: app_init.tex
\section{Initialization}
\label{sec:proof_init}

\subsection{Proof of Lemma~\ref{thm:init}}
Let $\delta$ denote the upper bound of $\norm{  p^{-1} \PP_{\Omega}(\Xstar)   -
    \Xstar}$ as in Lemma~\ref{lem:Chen_opnorm}, and let $\sigma_1 \geq \ldots \geq \sigma_n$ denote the singular values of
$p^{-1} \PP_{\Omega}(\Xstar)$. By Weyl's theorem, we have
\begin{equation} \abs{\sigma_i - \sstar_i} \leq \delta, \quad i \in [n].    \end{equation}
Note this implies $\sigma_{r+1} \leq \delta$, as $\sstar_{r+1} = 0$.

By definition, $Z^0 = [U; V] \Sigma^{\frac{1}{2}}$, where $U \Sigma V^\T$ is the
rank-$r$ SVD of $p^{-1} \PP_\Omega(\Xstar)$. According to
Lemma~\ref{lem:ZZ_to_X}, one has
\begin{equation}
    \label{eq:init_helper1}
    \begin{aligned}
        \norm{Z^0 {Z^0}^\T - \Zstar {\Zstar}^\T}_F 
        & \leq && 2 \norm{U \Sigma V^\T - \Xstar}_F \\
        & \stackrel{(a)}{\leq} && 2\sqrt{2r} \norm{U \Sigma V^\T - \Xstar} \\
        & \leq && 2\sqrt{2r} \left( \norm{ U \Sigma V^\T - p^{-1} \PP_{\Omega}(\Xstar) } + \norm{  p^{-1} \PP_{\Omega}(\Xstar) - \Xstar} \right) \\
        & \stackrel{(b)}{\leq} && 2\sqrt{2r} \left( \delta + \delta \right) \\
        &  = && 4\sqrt{2r} \delta, \\
    \end{aligned}
\end{equation}
where $(a)$ holds because
$\rank(U \Sigma V^\T - \Xstar)\leq 2r$,
$(b)$ holds since $\norm{ U \Sigma V^\T - p^{-1} \PP_{\Omega}(\Xstar) } = \sigma_{r+1}
\leq \delta $. 

Let $H = Z^0  - \Zbar^0$. We want to bound $d(Z^0, \Zstar)^2 = \norm{H}^2_F$. 
According to \eqref{eq:H'Z_symm}, $H^\T \Zbar^0$ is symmetric and ${Z^0}^\T {\Zbar^0}$ is positive semidefinite.
Hence we can write
\begin{equation}
\begin{aligned}
           & \hskip-30pt \norm{Z^0 {Z^0}^\T  - \Zstar {\Zstar}^\T}^2_F \\
         = &\norm{H {\Zbar^0}^\T + {\Zbar^0} H^\T + H H^\T}^2_F\\
         = & \trace \bigg ( (H^\T H)^2 + 2 (H^\T {\Zbar^0})^2 + 2 (H^\T H) ({\Zbar^0} ^\T {\Zbar^0}) + 4 (H^\T H) (H^\T {\Zbar^0}) \bigg) \\
         = & \trace \bigg( \left(H^\T H + \sqrt{2} H^\T \Zbar^0\right)^2 
             + (4 - 2\sqrt{2}) (H^\T H) (H^\T \Zbar^0) + 2 (H^\T H) ({\Zbar^0 }^\T \Zbar^0) \bigg)\\
         \geq & \trace \left((4-2\sqrt{2}) (H^\T H) (H^\T \Zbar^0) + 2 (H^\T H) ({\Zbar^0}^\T \Zbar^0)      \right) \\
         =  & (4-2\sqrt{2}) \trace \left( (H^\T H) ({Z^0}^\T \Zbar^0) \right) + (2\sqrt{2} - 2)  \norm{H {\Zbar^0}^T}^2_F, 
\end{aligned}
\end{equation}
where in the second line we used that $H^\T \Zbar^0$ is symmetric.
Besides, as ${Z^0}^\T \Zbar^0$ is positive semidefinite,
$(4-\sqrt{2})\trace((H^\T H) ({Z^0}^\T \Zbar^0) )$ is nonnegative. Therefore,
\begin{equation}
    \label{eq:init_helper2}
        \norm{Z^0 {Z^0}^\T - \Zstar {\Zstar}^\T}^2_F \geq (2\sqrt{2} - 2) \norm{ H {\Zbar^0}^\T}^2_F 
          \geq 4(\sqrt{2} - 1)  \sstar_\r  \norm{H}^2_F. \\
\end{equation}
Combining \eqref{eq:init_helper1} and \eqref{eq:init_helper2}, it
follows that 
\begin{equation}
         d(Z^0, \Zstar)^2 \leq \frac{ \norm{ Z^0 Z^0 - \Zstar {\Zstar}^\T}^2_F }{ 4(\sqrt{2} - 1) \sstar_r } 
         \leq 
         \frac{8 r }{ (\sqrt{2} - 1) \sstar_r } \delta^2.
\end{equation}
Therefore, it suffices to show 
\begin{equation}
    \begin{aligned}
         d(Z^0, \Zstar)^2 & \leq \frac{8 r }{ (\sqrt{2} - 1) \sstar_r } \delta^2\\
         & \stackrel{(a)}{=} c \frac{r}{\sstar_r} \left( \frac{\log n}{p} \norm{\Xstar}_\infty + \sqrt{\frac{\log n}{p}} \norm{\Xstar}_{\infty, 2} \right)^2  \\ 
         & \stackrel{(b)}{\leq} c \, r  \frac{{\sstar_1}^2}{\sstar_r} \left( \frac{\mu r \log
                 n}{p (\nmin)}  + \sqrt{\frac{\mu r \log n}{p (\nmin)}} \right)^2  \\ 
         & \leq \frac{1}{16}\sstar_r,
    \end{aligned}
\end{equation}
where in $(a)$ we replaced $\delta$ using Lemma~\ref{lem:Chen_opnorm},
and $(b)$ holds since by our incoherence assumption~\eqref{eq:defn_incoh} we have
\begin{align}
    &\norm{\Xstar}_\infty  = \norm{\Ustar \Sstar {\Vstar}^\T}_\infty 
            \leq \sstar_1 \max_{i,j} \norm{\Ustar_{(i)}} \norm{\Vstar_{(j)}} 
            \leq \sstar_1 \norm{\Ustar}_{2,\infty} \norm{\Vstar}_{2,\infty} 
            \leq \sstar_1 \, \frac{\mu r }{\nmin}, \\
    &\norm{\Xstar}_{\infty,2} = \norm{\Ustar \Sstar {\Vstar}^\T}_{\infty,2}
        \leq \sstar_1 \norm{\Ustar {\Vstar}^\T}_{\infty, 2}
        \stackrel{(c)}{\leq}
        \sstar_1 \sqrt{ \frac{\mu r}{\nmin}}.
\end{align}
Note that for $(c)$ we used $\norm{AB^\T}_{2,\infty} \leq \norm{A}_{2,\infty} \norm{B}$.

Hence, to obtain $d(Z^0, \Zstar)^2 \leq \frac{1}{16}\sstar_r $, it suffices to have
\begin{equation} p \geq \max\set{ \frac{c \mu r^{3/2} \kappa \log n}{\nmin}, \frac{c \mu r^2
\kappa^2\log n }{\nmin} } = \frac{c \mu r^2 \kappa^2\log n }{\nmin}. \end{equation}

Since $\PP_\C$ is just row-wise clipping, by Lemma~\ref{lem:contraction} we have
\[
d(Z^1, \Zstar)^2 \leq \norm{ \PP_\C ( Z^0 ) - \Zstar }^2_F \leq \norm{ Z^0 - \Zstar }^2_F.
\]

\subsection{Proof of Corollary~\ref{coro:proj_const}}
By the incoherence assumption, we have $\norm{\Zstar}_{2, \infty} \leq \sqrt{\frac{\mu r}{\nmin}
    \sstar_1}$, see \eqref{eq:Zbar_2_infty}.
It suffices to show $2 \sigma_1 \geq \sstar_1$. From the above discussion, we can
see that 
\[ \frac{8r}{(\sqrt{2} - 1 )\sstar_r} \delta^2 \leq \frac{1}{16} \sstar_r
    \Rightarrow \delta \leq \frac{1}{16} \sstar_r. \]
By Wely's theorem, we have
$| \sigma_1 - \sstar_1 | \leq \frac{1}{16}\sstar_r$. As a result, $2\sigma_1 \geq \sstar_1$.

%% file: app_regu.tex
\section{Regularity Condition}
\label{sec:proof_regu}
Analogous to the restricted strong convexity (RSC) and restricted strong
smoothness (RSS), we show that with high
probability our objective function $f$ satisfies the local curvature and local
smoothness conditions defined below.

\begin{itemize}
    \item \emph{Local Curvature Condition}

        \vspace{0.3cm}
        There exists constant $c_1, c_2$ such that for any $Z \in \C$ satisfying $d(Z,
        \Zstar) \leq \frac{1}{4} \sqrt{\sstar_r}$,
        \[ \ip{\nabla f(Z), H } \geq c_1 \norm{H}^2_F +  c_2 \norm{H^\T D \Zbar}^2_F. \]

    \item \emph{Local Smoothness Condition}

        \vspace{0.3cm}
        There exist constants $c_3, c_4$ such that for any $Z \in \C$ satisfying $d(Z,
        \Zstar) \leq \frac{1}{4} \sqrt{\sstar_r }$,
        \[ \norm{\nabla f(Z)}^2_F \leq c_3 \norm{H}^2_F + c_4\norm{H^\T D \Zbar}^2_F. \]
\end{itemize}

\subsection{Proof of the Local Curvature Condition}
\begin{equation}
    \begin{aligned}
          & \hspace{-0.5cm}\ip{\nabla f(Z), H} \\
          & = \frac{1}{p} \left(   \sum_{l = 1}^{2m} \ip{A_l, H {\Zbar}^\T + \Zbar H^\T + H H^\T }
          \cdot \ip{ ( A_l + A^\T_l ) (\Zbar + H), H} \right) + \lambda \trace(H^\T \Gamma) \\
          & \stackrel{(i)}{=} \frac{1}{p} \left( \sum_{l = 1}^{2m} \ip{A_l, H {\Zbar}^\T + \Zbar
                  H^\T + H H^\T } \cdot \ip{ A_l, H {\Zbar}^\T + \Zbar H^\T + 2
                  HH^\T } \right)
              + \lambda \trace(H^\T \Gamma) \\
          & =  \frac{1}{p} \bigg\{  \overbrace{ \sum_{l=1}^{2m} \ip{A_l, H {\Zbar}^\T + \Zbar
                  H^\T}^2 }^{a^2}  +  \overbrace{  \sum_{l=1}^{2m}   2 \ip{A_l,
                  H H^\T}^2 }^{b^2}  +  \sum_{l=1}^{2m} 3 \ip{A_l, H {\Zbar}^\T
                  + \Zbar H^\T} \ip{A_l, H H^\T} \bigg\} \\
          & \hspace{0.8cm} + \lambda \trace(H^\T \Gamma) \\
          & \stackrel{(ii)}{\geq} \frac{1}{p} \bigg\{   a^2 + b^2 - \frac{3}{\sqrt{2}}  \overbrace{ \sqrt{ \sum_{l=1}^{2m} \ip{A_l, H {\Zbar}^\T + \Zbar H^\T}^2 } }^a 
                                   \overbrace{ \sqrt{ \sum_{l=1}^{2m} 2 \ip{A_l,
                                               H H^\T}^2 } }^b  \bigg\} + \lambda \trace(H^\T \Gamma) \\
          & = \frac{1}{p} \bigg\{  \left(a - \frac{3}{2\sqrt{2}}b \right)^2 -
              \frac{1}{8}b^2 \bigg\}  + \lambda \trace(H^\T \Gamma)  \\
          & \stackrel{(iii)}{\geq} \frac{1}{p}\left( \frac{a^2}{2} - \frac{5}{4}b^2 \right) + \lambda \trace(H^\T \Gamma) \\
          & = \frac{1}{2} p^{-1} \norm{ \PP_{\OmegaY}( H {\Zbar}^\T + \Zbar H^\T ) }^2_F
                 - \frac{5}{2} p^{-1} \norm{\PP_{\OmegaY}( H H^\T )}^2_F + \lambda \trace(H^\T \Gamma).
    \end{aligned}
\end{equation}
where we used equation~\eqref{eq:trace_symm} for $(i)$,
the Cauchy-Schwarz inequality for $(ii)$, inequality $(a-b)^2 \geq \frac{a^2}{2} -
b^2$ for $(iii)$. Finally, in the last line we used $ \sum_{l=1}^{2m} \ip{A_l, M}^2 = \norm{\PP_{\OmegaY} ( M ) }^2_F  $.

We first lower bound $\frac{1}{2} p^{-1} \norm{ \PP_{\OmegaY}( H {\Zbar}^\T + \Zbar H^\T )
}^2_F$. By the symmetry of $\OmegaY$, it is equal to $ p^{-1} \norm{ \PP_{\Omega}( H_U
    {\Zbar}^\T_V + \Zbar_U H^\T_V ) }^2_F $, which expands to
\begin{equation}
    \begin{aligned}
    p^{-1} \norm{ \PP_{\Omega}( H_U {\Zbar}^\T_V) }^2_F  
            + p^{-1} \norm{ \PP_{\OmegaY}(\Zbar_U H^\T_V ) }^2_F  
            + 2 p^{-1} \ip{ \PP_{\Omega}( H_U {\Zbar}^\T_V) ,  \PP_{\Omega}(\Zbar_U H^\T_V)   }.
    \end{aligned}
\end{equation}
As both $H_U {\Zbar}^\T_V$ and $\Zbar_U H^\T_V$ belong to $T$, we use
Lemma~\ref{lem:rsc_chen} to lower bound above three terms, respectively. This gives us
\begin{equation}
    \begin{aligned}
        & \hskip-30pt \frac{1}{2} p^{-1} \norm{ \PP_{\OmegaY}( H {\Zbar}^\T + \Zbar H^\T ) }^2_F \\
      \geq & \; (1-\delta) \left(  \norm{ H_U {\Zbar}^\T_V }^2_F 
            + \norm{ \Zbar_U H^\T_V }^2_F  \right)
            + 2 \ip{ H_U {\Zbar}^\T_V ,  \Zbar_U
                H^\T_V } - 2 \delta \norm{ H_U {\Zbar}^\T_V }_F \norm{\Zbar_U
                H^\T_V}_F \\
     \geq & \; (1-\delta) \left(  \norm{ H_U {\Zbar}^\T_V }^2_F 
            + \norm{ \Zbar_U H^\T_V }^2_F  \right)
            + 2 \ip{ H_U {\Zbar}^\T_V ,  \Zbar_U H^\T_V } - \delta \left( \norm{ H_U {\Zbar}^\T_V }^2_F 
            + \norm{\Zbar_U H^\T_V}^2_F \right) \\
     \stackrel{(iv)}{\geq} & (1- 2\delta) \sstar_r \left(  \norm{ H_U }^2_F + \norm{H_V }^2_F \right)
                             + 2 \ip{ H_U {\Zbar}^\T_V ,  \Zbar_U H^\T_V }\\
     = & \; (1- 2\delta) \sstar_r \norm{ H }^2_F + 2 \ip{ H_U {\Zbar}^\T_V ,  \Zbar_U H^\T_V }.
    \end{aligned}
\end{equation}
where we used $\norm{ H_U {\Zbar}^\T_V }^2_F \geq \sstar_r \norm{H_U}^2_F$ and
$\norm{ \Zbar_U H^\T_V }^2_F \geq \sstar_r \norm{H_V}^2_F$ for $(iv)$.

Until now, we obtain
\begin{equation}
     \begin{aligned}
         \ip{\nabla f(Z), H} 
          & \geq (1 - 2\delta)\sstar_r  \norm{H}^2_F 
                 + 2 \ip{H_U {\Zbar}^\T_V, \Zbar_U H^\T_V} + \lambda \trace(H^\T \Gamma)
                    - \frac{5}{2} p^{-1} \norm{\PP_{\OmegaY}( H H^\T )}^2_F. 
      \end{aligned}
\end{equation}

Next, we lower bound $2 \ip{H_U {\Zbar}^\T_V, \Zbar_U H^\T_V} + \lambda
\trace(H^\T \Gamma)$ together. Rewriting
\begin{equation}
    \label{eq:curve_helper1}
    \begin{aligned}
    &2 \ip{H_U {\Zbar}^\T_V, \Zbar_U H^\T_V} = \ip{ H,
        \begin{bmatrix}
            0 & \Zbar_U H^\T_V \\
            \Zbar_V H^\T_U & 0 \\
        \end{bmatrix} \Zbar }
    =  \ip{ H , \frac{1}{2} ( \Zbar H^\T -  D \Zbar H^\T D ) \Zbar}, \\
    & ZZ^\T - \Zbar{\Zbar}^\T = HH^\T + \Zbar H^\T + H {\Zbar}^\T,
    \end{aligned}
\end{equation}
and plugging in $\Gamma = D Z Z^\T D Z$, we then have
\begin{align}
    \begin{aligned}
   & && \hskip-30pt 2 \ip{H_U {\Zbar}^\T_V, \Zbar_U H^\T_V} + \lambda \trace(H^\T \Gamma) \\
   & = &&  \ip{H,  \frac{1}{2} ( \Zbar H^\T -  D \Zbar H^\T D  )\Zbar} +  \lambda \ip{ H, D(ZZ^\T - \Zbar {\Zbar}^\T)D Z} 
        + \lambda \ip{ H, D(\Zbar {\Zbar}^\T)D \Zbar} \\
   & && + \lambda \ip{ H, D(\Zbar {\Zbar}^\T)D  H} \\
   & \stackrel{(a)}{=} &&  \ip{H,  \frac{1}{2} ( \Zbar H^\T -  D \Zbar H^\T D   )  \Zbar} +  \lambda \ip{ H, D(ZZ^\T - \Zbar {\Zbar}^\T)D Z} 
                            + \lambda \norm{{\Zbar}^\T D H}^2_F \\
   & \stackrel{(b)}{=} &&   \lambda \norm{ {\Zbar}^\T D H}^2_F +  \ip{H,  \frac{1}{2} ( \Zbar H^\T -  D \Zbar H^\T D  ) \Zbar 
                            + \lambda  D(HH^\T + \Zbar H^\T + H {\Zbar}^\T)D (\Zbar + H) }  \\
   & \stackrel{(c)}{=} &&    \lambda \norm{ {\Zbar}^\T D H }^2_F + \frac{1}{2}
                            \norm{H^\T \Zbar}^2_F
                            + \lambda \norm{H^\T D H}^2_F +  3 \lambda \trace(H^\T D H H^\T D \Zbar) \\
   & &&                     + \left(\lambda - \frac{1}{2}\right) \trace(H^\T D \Zbar H^\T D \Zbar)  \\
   & = &&\frac{\lambda}{2} \norm{ {\Zbar}^\T D H }^2_F +
         \frac{\lambda}{2} \norm{ {\Zbar}^\T D H + 3 H^\T D H }^2_F -\frac{7}{2} \lambda \norm{H^\T D H}^2_F\\
   & &&  + \frac{1}{2} \norm{H^\T Z}^2_F + \left(\lambda - \frac{1}{2}\right) \trace(H^\T D \Zbar H^\T D \Zbar)  \\
   & \geq && \frac{\lambda}{2} \norm{ {\Zbar}^\T D H }^2_F - \frac{7}{2}\lambda
         \norm{H}^4_F + \left(\lambda - \frac{1}{2}\right) \trace(H^\T D \Zbar H^\T D \Zbar)\\
    \end{aligned}
\end{align}
Equality $(a)$ holds because ${\Zbar}^\T D \Zbar = 0$. We plug in
\eqref{eq:curve_helper1} in $(b)$. For $(c)$, we use ${\Zbar}^\T D \Zbar = 0$ and that $H^\T \Zbar$ is
symmetric. 
Finally, we take $\lambda = \frac{1}{2}$ and use Lemma~\ref{lem:PP_HH} to upper bound
$p^{-1} \norm{\PP_{\OmegaY}( H H^\T )}^2_F$:
\begin{equation}
    \begin{aligned}
         \ip{\nabla f(Z), H} 
          & \geq &&  (1-2\delta) \sstar_r \norm{H}^2_F + \frac{1}{4} \norm{ {\Zbar}^\T D H }^2_F - \frac{7}{4} \norm{H}^4_F
                    - \frac{5}{2} (1+\delta) \norm{H}^4_F - \frac{5}{2} \delta \sstar_r \norm{H}^2_F\\
                    & = && \left( (1- 2\delta)\sstar_r - \frac{5}{2}(\frac{17}{10}+\delta) \norm{H}^2_F - \frac{5}{2}\delta \sstar_r \right)  \norm{H}^2_F
                    + \frac{1}{4} \norm{{\Zbar}^\T D H}^2_F. \\
          \end{aligned}
\end{equation}

For simplicity, we take $\delta = \frac{1}{16}$. We also
have $\norm{H}^2_F \leq \frac{1}{16} \sstar_r$. This leads to
\begin{equation}
    \label{eq:local_curvature}
    \begin{aligned}
    \ip{\nabla f(Z), H}
     \geq \frac{227}{512} \sstar_r \norm{H}^2_F + \frac{1}{4} \norm{ {\Zbar}^\T D H}^2_F.
    \end{aligned}
\end{equation}

Note that this lower bound holds with high probability uniformly for all $Z \in \C$ such that $d(Z, \Zstar)
\leq \frac{1}{4} \sqrt{\sstar_r}$, since Lemma~\ref{lem:PP_HH} and \ref{lem:rsc_chen} hold uniformly.

When the ground truth $\Xstar$ is positive semidefinite,
we don't need to do lifitng nor impose the regularizer. Using
Lemma~\ref{lem:rsc_chen}, we can lower bound
 $ \frac{1}{2} p^{-1} \norm{ \PP_{\OmegaY}( H {\Zbar}^\T + \Zbar H^\T ) }^2_F \gtrsim (1-\delta)\sstar_r \norm{H}^2_F$ directly. 
Taking proper constants, we can obtain the standard restricted strong convexity
condition:
\[ \ip{\nabla f(Z), H} \gtrsim \sstar_r \norm{H}^2_F. \]

\subsection{Proof of the Local Smoothness Condition}
To upper bound $\norm{\nabla f(Z)}^2_F = \max_{\norm{W}_F = 1} \abs{ \ip{\nabla f(Z), W} }^2$, 
it suffices to show that for any $n \times r$ $W$ of unit Frobenius norm, 
$\abs{ \ip{\nabla f(Z), W}}^2 $ is upper bounded. We first write
\begin{equation}
    \begin{aligned}
           & \hskip-10pt \ip{\nabla f(Z), W} \\
         = &\; \frac{1}{p}\sum_{l = 1}^{2m}\left( \ip{A_l, H {\Zbar}^\T + \Zbar H^\T} + \ip{A_l, H H^\T } \right) 
                                \cdot \ip{ ( A_l + A^\T_l ) (\Zbar + H), W} + \lambda \trace(W^\T  \Gamma)\\ 
         \stackrel{(i)}{=} & \; \frac{1}{p}\sum_{l = 1}^{2m} \left( \ip{A_l, H {\Zbar}^\T + \Zbar H^\T} + \ip{A_l, H H^\T } \right)
               \left( \ip{A_l, W {\Zbar}^\T + \Zbar W^\T} +  \ip{A_l, W H^\T + H W^\T} \right) \\
           & \hspace{0.2cm} + \lambda \trace(W^\T  \Gamma)  \\
         = & \; \frac{1}{p}\bigg\{ \ip{ \PP_{\OmegaY}( H {\Zbar}^\T + \Zbar H^\T ), \PP_{\OmegaY}( W {\Zbar}^\T + \Zbar W^\T ) } 
                               + \ip{ \PP_{\OmegaY}( H H^\T ), \PP_{\OmegaY}( W {\Zbar}^\T + \Zbar W^\T ) }\\
           & \hspace{0.2cm}  + \ip{ \PP_{\OmegaY}( H {\Zbar}^\T + \Zbar H^\T ),  \PP_{\OmegaY}( W H^\T + H W^\T ) } 
                               + \ip{ \PP_{\OmegaY}( H H^\T ), \PP_{\OmegaY}( W H^\T + H W^\T ) } \bigg\} \\
           & \hspace{0.2cm}  + \lambda \trace(W^\T  \Gamma),  \\
    \end{aligned}
\end{equation}
where we used \eqref{eq:trace_symm} for $(i)$. Since $(a+b+c+d+e)^2 \leq 5(a^2+b^2+c^2+d^2+e^2)$, we have
\begin{equation}
    \begin{aligned}
               & \hskip-10pt \abs{ \ip{\nabla f(Z), W} }^2    \\
         \leq & \; \frac{5}{p^2} \bigg\{ \ip{ \PP_{\OmegaY}( H {\Zbar}^\T + \Zbar H^\T ), \PP_{\OmegaY}( W {\Zbar}^\T + \Zbar W^\T ) }^2 
                  + \ip{ \PP_{\OmegaY}( H H^\T ), \PP_{\OmegaY}( W {\Zbar}^\T + \Zbar W^\T ) }^2 \\
              & + \ip{ \PP_{\OmegaY}( H {\Zbar}^\T + \Zbar H^\T ),  \PP_{\OmegaY}( W H^\T + H W^\T ) }^2 
                  + \ip{ \PP_{\OmegaY}( H H^\T ), \PP_{\OmegaY}( W H^\T + H W^\T ) }^2  \bigg\} \\
              & + 5 \lambda ^2 \trace(W^\T  \Gamma)^2  \\
         \stackrel{(ii)}{\leq} & \frac{5}{p^2} \left( \norm{ \PP_{\OmegaY}( H {\Zbar}^\T + \Zbar H^\T ) }^2_F  
                  + \norm{ \PP_{\OmegaY}(H H^\T) }^2_F  \right)   \\
              & \hspace{0.2cm} \cdot \left( \norm{ \PP_{\OmegaY}( W {\Zbar}^\T + \Zbar W^\T ) }^2_F  
                  + \norm{ \PP_{\OmegaY}( W H^\T + H W^\T ) }^2_F  \right) + 5
              \lambda^2 \norm{ \Gamma}^2_F \overbrace{\norm{W}^2_F}^{=1}  \\
         \stackrel{(iii)}{\leq} & \frac{5}{p} \bigg( 2 \overbrace{ \norm{ \PP_{\OmegaY}( H
                    {\Zbar}^\T) }^2_F }^{\circled{1}}  + 2 \overbrace{ \norm{
                    \PP_{\OmegaY} (\Zbar H^\T ) }^2_F }^{\circled{2}}
                  + \overbrace{\norm{ \PP_{\OmegaY}(H H^\T) }^2_F
                  }^{\circled{3}} \bigg)  \\
              & \hspace{0.2cm} \cdot \frac{1}{p} \bigg( 
                    2 \overbrace{ \norm{ \PP_{\OmegaY}( W {\Zbar}^\T )}^2_F}^{\circled{4}}   
                  + 2 \overbrace{ \norm{ \PP_{\OmegaY}(\Zbar W^\T )   }^2_F}^{\circled{5}}
                  + 2 \overbrace{ \norm{ \PP_{\OmegaY}( W H^\T )      }^2_F}^{\circled{6}}
                  + 2 \overbrace{ \norm{ \PP_{\OmegaY}( H W^\T )      }^2_F}^{\circled{7}}\bigg)\\
              &  + 5 \lambda^2 \norm{ \Gamma}^2_F,
    \end{aligned}  
\end{equation}
where we used the Cauchy-Schwarz inequality for $(ii)$, and $(a+b)^2 \leq 2 (a^2
+ b^2)$ for $(iii)$. We then use Lemma~\ref{lem:upper_PP_Y_AB'} to upper bound
\circled{1}, \circled{2}, \circled{4}, \circled{5}, \circled{6}, \circled{7},
and Lemma~\ref{lem:PP_HH} for \circled{3}. Also since 
$\norm{W}_F =1$, one has
\begin{equation}
    \label{eq:smooth_helper1}
    \begin{aligned}
        &      && \hskip-20pt \abs{ \ip{\nabla f(Z), W} }^2 \\
        & \leq && 5 \left( 8n \norm{ H }^2_F \norm{\Zbar}^2_{2,\infty} + (1+\delta) \norm{H}^4_F + \delta \sstar_r \norm{H}^2_F  \right)
                     \cdot  \left( 8n  \norm{\Zbar}^2_{2, \infty} + 8n \norm{H}^2_{2, \infty} \right) \\
        &      && \hspace{0.2cm} + 5 \lambda^2 \norm{\Gamma}^2_F \\
        &  = && 40n \left( 8n \norm{\Zbar}^2_{2,\infty} + (1+\delta) \norm{H}^2_F + \delta \sstar_r   \right)\norm{ H }^2_F 
                     \cdot  \left( \norm{\Zbar}^2_{2, \infty} + \norm{H}^2_{2, \infty} \right) 
                     + 5 \lambda^2 \norm{\Gamma}^2_F \\
        & \leq  && 400 \mu r \sstar_1 \left( 8 \mu r \sstar_1 + (1+\delta) \norm{H}^2_F + \delta \sstar_r \right)\norm{ H }^2_F 
        + 5 \lambda^2 \norm{\Gamma}^2_F, 
    \end{aligned}
\end{equation}
where in the last line we plugged in $\norm{\Zbar}_{2, \infty} \leq \sqrt{\dfrac{\mu r}{n} \sstar_1}$
and $\norm{H}_{2,\infty} \leq 3 \sqrt{\dfrac{\mu r}{n} \sstar_1}$, i.e. \eqref{eq:Zbar_2_infty} and \eqref{eq:H_2_infty}.

Next, we bound
\begin{equation}
    \label{eq:smooth_helper2}
    \begin{aligned}
        \norm{\Gamma}^2_F 
        & =  \norm{ D(ZZ^\T - \Zbar {\Zbar}^\T)DZ + D \Zbar {\Zbar}^\T D Z}^2_F \\
        & \leq  2 \norm{D(ZZ^\T - {\Zbar}{\Zbar}^\T) D Z}^2_F + 2 \norm{D{\Zbar}{\Zbar}^\T D Z}^2_F \\
        & \stackrel{(a)}{\leq}  2 \norm{ZZ^\T - {\Zbar}{\Zbar}^\T}^2_F \norm{Z}^2  + 2 \norm{\Zbar}^2 \norm{ {\Zbar}^\T D Z}^2_F \\
        & \stackrel{(b)}{=} 2 \norm{HH^\T + \Zbar H^\T + H{\Zbar}^\T}^2_F \norm{Z}^2  + 2 \norm{\Zbar}^2 \norm{ {\Zbar}^\T D H}^2_F \\
        & \leq  6 \left( \norm{HH^\T}^2_F + \norm{\Zbar H^\T}^2_F + \norm{ H{\Zbar}^\T}^2_F \right) \norm{Z}^2  
                  + 2 \norm{\Zbar}^2 \norm{ {\Zbar}^\T D H}^2_F \\
        & \stackrel{(c)}{\leq}  6 \left( \norm{H}^2_F + 2 \norm{\Zbar}^2 \right) \norm{H}^2_F \norm{Z}^2  
                  + 2 \norm{\Zbar}^2 \norm{ {\Zbar}^\T D H}^2_F \\
        &  \stackrel{(d)}{=}   6 \left( \norm{H}^2_F + 4 \sstar_1 \right) \norm{H}^2_F \norm{Z}^2  
                  + 4 \sstar_1 \norm{ {\Zbar}^\T D H}^2_F. \\
    \end{aligned}  
\end{equation}
Inequality $(a)$ holds because $\norm{AB}_F \leq \norm{A}\norm{B}_F$ and $\norm{D} = 1$. To
get $(b)$, for the first term in the 3rd line we expand $ZZ^\T - \Zbar {\Zbar}^\T$, for the second
term we expand $Z = \Zbar + H$ and use ${\Zbar}^\T D  \Zbar = 0$. 
For $(c)$, we use $\norm{AB}_F \leq \norm{A}\norm{B}_F \leq \norm{A}_F
\norm{B}_F$. Last, $(d)$ holds because $\norm{\Zbar}^2 = 2 \sstar_1$.

Finally, we combine \eqref{eq:smooth_helper1} and \eqref{eq:smooth_helper2}.
As before, take $\lambda = \frac{1}{2}$,  $\delta = \frac{1}{16}$, and
$\norm{H}^2_F \leq \frac{1}{16} \sstar_r$, we obtain
 \begin{equation}
    \label{eq:smooth}
    \begin{aligned}
               & \hskip-20pt \norm{\nabla f (Z)}^2_F \\
          \leq & \left\{  400  \mu r \sstar_1  \left( 8 \mu r \sstar_1 +
                           (1+\delta)\norm{H}^2_F  +  \delta \sstar_r \right) 
                      + 30 \lambda^2 \left( \norm{H}^2_F + 4 \sstar_1 \right)
                      \norm{Z}^2  \right \} \norm{H}^2_F \\
               &    + 20 \lambda^2 \sstar_1 \norm{ {\Zbar}^\T D H}^2_F\\
          \stackrel{(a)}{\leq} & \left\{  400  \mu r \sstar_1  \left( 8 \mu r \sstar_1 +
                (1+\delta)\norm{H}^2_F  +  \delta \sstar_r \right) 
                      + \frac{735}{8} \sstar_1 \lambda^2 \left( \norm{H}^2_F + 2 \sstar_1 \right)
                       \right \} \norm{H}^2_F \\
               &    + 20 \lambda^2 \sstar_1 \norm{ {\Zbar}^\T D H}^2_F\\
          \stackrel{(b)}{\leq} &    \left\{  400   \left( 8  + \frac{17}{256}
                +  \frac{1}{16} \right) + \frac{735}{32} \left( \frac{1}{16}+ 2\right)
       \right \}  \mu^2 r^2 {\sstar_1}^2   \norm{H}^2_F 
                  + 5 \sstar_1 \norm{ {\Zbar}^\T D H}^2_F\\
          \leq & \; 3299  \mu^2 r^2 {\sstar_1}^2 \norm{H}^2_F 
                  + 5 \sstar_1 \norm{ {\Zbar}^\T D H}^2_F, \\
    \end{aligned}
\end{equation}
where for $(a)$ we used
$\norm{Z} \leq \norm{H} + \norm{\Zbar} \leq \frac{1}{4}\sqrt{\sstar_r} +
\sqrt{2\sstar_1} \leq \frac{7}{4}\sqrt{\sstar_1}$, for $(b)$ we used 
$\mu, r \geq 1$.

As before, this condition holds uniformly for all $Z$ such that $d(Z, \Zstar)
\leq \frac{1}{4}\sqrt{\sstar_r}$ and satisfying the incoherence condition.

For the case $\Xstar$ is positive semidefinite,  as we don't need
to impose the regularizer, standard restricted strong smoothness condition follows:
\[ \norm{\nabla f(Z)}^2_F \lesssim \sstar_1 \norm{H}^2_F.  \]

\subsection{Proof of Lemma~\ref{thm:rc}}
Rearranging the terms in the smoothness condition \eqref{eq:smooth}, we can further bound
\begin{equation}
    \label{eq:local_smooth}
    \begin{aligned}
      \frac{1}{4} \norm{ {\Zbar}^\T D H}^2_F
      & \geq \frac{ \norm{\nabla f(Z)}^2_F} {20 \mu^2 r^2 \kappa \sstar_1} - \frac{3299}{20} \sstar_r \norm{H}^2_F \\
      & \geq \frac{ \norm{\nabla f(Z)}^2_F} {13196 \mu^2 r^2 \kappa \sstar_1} -
      \frac{128}{512} \sstar_r \norm{H}^2_F.\\
    \end{aligned}
\end{equation}
Combining equation \eqref{eq:local_curvature} and \eqref{eq:local_smooth}, it
follows that
\begin{equation}
    \label{eq:regularity}
        \ip{\nabla f(Z), H} \geq  
     \frac{99}{512} \sstar_r \norm{H}^2_F + \frac{1}{13196 \, \mu^2 r^2 \kappa \sstar_1 } \norm{ \nabla f (Z) }^2_F.
\end{equation}

Finally, by upper bounding the probability that
Lemma~\ref{lem:PP_HH},~\ref{lem:upper_PP_Y_AB'}, or~\ref{lem:rsc_chen} fails,
and the sample probability $p$ these lemmas require,
we conclude that once
\begin{equation}
    p \geq c \max\left(\frac{\mu r \log n}{\nmin}, \frac{\mu^2 r^2 \kappa^2}{\nmin}\right),
\end{equation}
regularity condition \eqref{eq:regularity} holds 
with probability at least $1 - c_1 n^{-c_2}$, where $c, c_1, c_2$ are constants.

%% file: app_conv.tex
\section{Linear Convergence}
\subsection{Proof of Lemma~\ref{thm:linear_conv}}
Let $H^k = Z^k - \Zbar^k$.  Our iterate is $Z^{k+1} = \PP_\C( Z^k - \eta \nabla f(Z^k) ) $.
Since $\PP_\C$ is just row-wise clipping, by Lemma~\ref{lem:contraction} we have
\begin{equation}
    \norm{ \PP_\C \left( Z^k - \frac{\eta}{\sstar_1} \nabla f(Z^k) \right) -  {\Zbar}^k }^2_F  
    \leq \norm{ Z^k - \frac{\eta}{\sstar_1} \nabla f(Z^k)  -  {\Zbar}^k }^2_F.
\end{equation}
It follows that 
\begin{equation}
\begin{aligned}
         & \hskip-30pt \norm{Z^{k+1} - \Zbar^k }^2_F\\
    \leq & \; \norm{ Z^k - \frac{\eta}{ \sstar_1 }\nabla f(Z^k) - \Zbar^k}^2_F \\
      =  & \; \norm{H^k}^2_F + \frac{\eta^2}{ {\sstar_1}^2 }\norm{ \nabla f(Z^k)}^2_F -  \frac{2\eta}{\sstar_1} \ip{ \nabla f(Z^k), H^k} \\
    \stackrel{(a)}{\leq} & \;  \norm{H^k}^2_F + \frac{\eta^2}{{\sstar}^2_1 } \norm{ \nabla f(Z^k)}^2_F 
               - \frac{2\eta}{\sstar_1}\left(\frac{1}{\alpha} \sstar_r \norm{H^k}^2_F 
               + \frac{1}{\beta \sstar_1 }\norm{\nabla f(Z^k)}^2_F \right)\\
      = & \;  \left( 1 - \frac{2\eta}{\alpha \kappa } \right) \norm{H^k}^2_F 
              + \frac{\eta (\eta - 2/\beta)}{ {\sstar}^2_1 } \norm{\nabla f(Z^k)}^2_F\\
      \stackrel{(b)}{\leq}  & \; \left( 1 - \frac{2\eta}{\alpha \kappa } \right) \norm{H^k}^2_F,
\end{aligned}
\end{equation}
where we use the definition of $RC(\epsilon, \alpha, \beta)$ for $(a)$
and $0 < \eta \leq \min\set{ \alpha / 2,  2 / \beta }$ for $(b)$.
Therefore, 
\begin{equation}
    d(Z^{k+1}, \Zstar) = \min_{\Ztilde \in \tS} \norm{Z^{k+1} - \Ztilde}^2_F \leq \sqrt{1 - \frac{2\eta}{\alpha \kappa} }d(Z^k, \Zstar).
\end{equation}